\documentclass[journal]{IEEEtran}
\normalsize

\usepackage{amsmath,amssymb,amsfonts} 
\allowdisplaybreaks
\usepackage{graphicx}
\usepackage{breqn}
\usepackage{cite}
\usepackage{textcomp}
\usepackage{lipsum}  

\usepackage{amsthm, bm, bbm}

\usepackage[font=small,labelfont=bf]{caption}
\usepackage{pdfpages, float}
\usepackage{mathrsfs}
\usepackage{graphicx}
\newtheorem{thm}{Theorem}

\newtheorem{rem}{Remark}
\newtheorem{assumption}{Assumption}

\newtheorem{definition}{Definition}

\newtheorem{example}{Example}

\newtheorem{prop}{Proposition}
\usepackage{subcaption}
\definecolor{blu}{RGB}{0, 102, 204}
\definecolor{purp}{RGB}{128,0,128}
\definecolor{rd}{RGB}{255,69,0}
\definecolor{org}{RGB}{255, 95, 31}
\definecolor{cyn}{RGB}{0, 200, 200}

\def\comment#1{\textcolor{purp}{#1}}

\usepackage[linesnumbered,ruled,vlined]{algorithm2e}
\usepackage{algpseudocode}
\usepackage{setspace}

\SetCommentSty{mycommfont}

\date{}
\begin{document}

\title{
Stabilization of Perturbed Loss Function: Differential Privacy without Gradient Noise
	\thanks{The material is based upon work supported in part by the National Science Foundation (NSF) and Office of the Under Secretary of Defense (OUSD) Research and Engineering, ITE2326898, and ITE2226447 as part of the NSF Convergence Accelerator Track G: Securely Operating Through 5G Infrastructure Program. The simulation code used in this work is available at https://github.com/theorycoder/distributedautoencoders.	
	}
}
\begin{tiny}
	\author{
	\IEEEauthorblockN{
	Salman Habib$^\dagger$, R\'{e}mi A. Chou$^\ast$, and Taejoon Kim$^\dagger$\\
		}
	\IEEEauthorblockA{
	\vspace{0.15cm}
	\begin{small}
	$^\ast$Department of Computer Science and Engineering, University of Texas at Arlington \\
		$^\dagger$School of Electrical, Computer and Energy Engineering, Arizona State University \\
	\end{small}
	}
}
\end{tiny}

\maketitle 
\vspace{-2cm}

\begin{abstract}

We propose SPOF (\underline{S}tabilization of \underline{P}erturbed L\underline{o}ss \underline{F}unction),
a differentially private training mechanism intended for multi-user local differential privacy (LDP). SPOF perturbs a stabilized Taylor expanded polynomial approximation of a model's training loss function, where each user's data is privatized by calibrated noise added to the coefficients of the polynomial. Unlike gradient-based mechanisms such as differentially private stochastic gradient descent (DP-SGD), SPOF does not require injecting noise into the gradients of the loss function, which improves both computational efficiency and stability. This formulation naturally supports simultaneous privacy guarantees across all users. Moreover, SPOF exhibits robustness to environmental noise during training, maintaining stable performance even when user inputs are corrupted. We compare SPOF with a multi-user extension of DP-SGD, evaluating both methods in a wireless body area network (WBAN) scenario involving heterogeneous user data and stochastic channel noise from body sensors. Our results show that SPOF achieves, on average, up to $3.5\%$ higher reconstruction accuracy and reduces mean training time by up to $57.2\%$ compared to DP-SGD, demonstrating superior privacy-utility trade-offs in multi-user environments.
\end{abstract}

\section{Introduction}
\label{sec:Intro}
To address the open challenge of achieving differential-privacy (DP) in multi-user learning systems, we introduce SPOF (\underline{S}tabilization of \underline{P}erturbed L\underline{o}ss \underline{F}unction), a DP approach that incorporates local differential-privacy (LDP). SPOF ensures that each user’s data is privatized independently, without any aggregation, by injecting calibrated noise directly into the coefficients of a Taylor expanded polynomial approximation of a model's loss function. The coefficients of this loss function encode each user's contribution, which {are} perturbed locally using Laplace noise. This formulation avoids the need for perturbing gradients at every optimization step, as required in differentially private stochastic gradient descent (DP-SGD), leading to improved computational efficiency and training stability. Multi-user privacy is crucial in distributed systems where multiple data sources independently contribute {to} sensitive information. Examples include wireless body area networks (WBANs), where physiological signals from body sensors are wirelessly transmitted to a central server \cite{wban1,wban2,jsan11040067}, smart cities \cite{zanella2014internet}, and collaborative filtering \cite{calandrino2011privacy}. In such applications, protecting individual user data from inference attacks requires mechanisms that guarantee privacy jointly over all users. \\
\indent Several mechanisms have been proposed to enhance DP. For instance, \cite{feng2022inferential} enforces privacy via inferential separation, whereas \cite{optimalDP} derives optimal perturbation strategies that minimize utility loss under fixed privacy budgets. The authors of \cite{wei2020federated} added artificial DP noise to client-side model parameters before aggregation, whereas \cite{hu2022one} defends against membership and inversion attacks by privatizing and normalizing model confidence scores under a tunable budget. While prior approaches like the functional mechanism (FM) \cite{fm1,fm2,fm3} perturb the objective function instead of the gradients, they are inherently designed for single-user settings and lack stabilization techniques. In contrast, SPOF {adds a loss stabilization constant to the model parameters} that enhance the model's prediction accuracy. Moreover, SPOF generalizes to any number of users $m$. Thus, FM naturally arises as a special case of SPOF in the single-user scenario $(m=1)$ without loss stabilization. \\
\indent We study SPOF using autoencoders due to their ubiquity in practical systems and their analytically tractable loss function, which allows for theoretical derivations of privacy guarantees. To accommodate a realistic multi-user setting, we adopt a distributed autoencoder (DA) architecture, where $m$ users each possess an independent encoder and share a common decoder that reconstructs the original data. DAs have gained attention for their efficiency in distributed learning systems, including schemes such as distributed image compression \cite{drasic} and task-aware source coding \cite{li2023taskaware}. The DA framework enables multi-user privacy, as each user’s input contributes independently to the local privacy budget, making it a natural choice for providing LDP privacy guarantees. In contrast, FM, is suitable only for standard autoencoders (SAs) that consist of a single encoder and a decoder \cite{fm2}. \\
\indent For simulations, we apply SPOF to a DA-based wireless body area network {(e.g. WBAN)} in which multiple encoders are employed to compress physiological data from the Fitbit dataset \cite{fitbit}, where environmental noise from body sensors and channels introduces realistic training data corruption \cite{noise,noise2}. A common decoder (e.g., hospital) collects all the compressed signals for post processing. As a benchmark for SPOF, we extend DP-SGD \cite{dpsgd1,dpsgd2,dpsgd3,dpsgd4,dpsgd5} to the multi-user setting, by incorporating advanced techniques such as bookkeeping \cite{small_cost} and group clipping \cite{group_wise}. While recent methods such as DP based on zeroth order optimization (DP-ZO) \cite{DP_ZO} also apply DP at the loss level, they lack the analytical tractability of SPOF, and whether DP-ZO leverages environmental noise for privacy is unknown. In this paper, we consider privacy mechanisms for SPOF and DP-SGD, both using Laplace noise. Notably, the use of Laplace noise in the context of DP-SGD has also been explored in \cite{DP_ZO}. \\
\indent Our analysis shows that SPOF satisfies DP under both noiseless and noisy input conditions, where the latter is caused due to environmental noise. The proposed loss stabilization technique biases the approximated loss of a DA {and improves the model's prediction accuracy.} We further derive sensitivity expressions and show that SPOF can yield lower sensitivity than DP-SGD {for a certain value of the Taylor-approximated loss function parameter,} enabling lower noise injection for a fixed privacy budget. \\
\indent Simulations using the Fitbit dataset confirm that SPOF outperforms DP-SGD in both accuracy and training time. Notably, {when environmental noise standard deviation (s.d.) is large enough,} SPOF requires less additional DP noise with a probability of at least $0.5$ due to a multiplicative scaling factor that reduces the overall noise variance. On the other hand, {our simulations indicate that DP-SGD suffers utility degradation under noisy conditions, presumably due to large fluctuations in gradient magnitudes during early training epochs.} \\
Our main contributions are as follows:
\begin{itemize}
  \item We derive a multi-user cross-entropy loss function that depends on all $m+1$ trainable parameters, i.e., weights of the $m$ encoders and a decoder of the DA, and approximate it using a Taylor-expanded loss $\tilde{L}$. SPOF perturbs the coefficients of $\tilde{L}$, while DP-SGD perturbs its gradients. In Section \ref{sec:DA_noiseless} and \ref{sec:loss_dpsgd}, we show that both mechanisms satisfy DP guarantees under noiseless and noisy input conditions.
  \item In Section \ref{sec:DA_lossStab}, we discuss an essential component of SPOF that involves adding a loss stabilization constant into a {model's learnable parameters}. This constant {stabilizes the direction of gradient updates and improves a DA's convergence and reconstruction accuracy,} without compromising the privacy guarantee.
  \item We analyze the effect of environmental noise $\mathbf{n}$ on SPOF in Section \ref{sec:DA_noisy_loss}. We show analytically that, with probability at least $0.5$, environmental noise reduces the required DP noise magnitude in SPOF. This joint analysis involving probability distributions of model parameters and environmental noise engenders accurate DP noise calibration in presence of noisy inputs.
  \item Section \ref{sec:loss_dpsgd} extends DP-SGD {to a} multi-user scenario. Through derivation of sensitivity expressions for SPOF and DP-SGD under multi-user input changes, Appendix~\ref{app:sensitivity_optimize} shows that sensitivity minimization through Taylor approximated loss function parameter tuning can reduce DP noise magnitude in SPOF. We also show scenarios where SPOF achieves lower sensitivity than DP-SGD.
  \item We evaluate SPOF on a WBAN using Fitbit data for both noiseless and noisy input conditions in Section \ref{sec:simulation}. Compared to DP-SGD, SPOF achieves up to $3.5\%$ higher reconstruction accuracy and up to $57.2\%$ reduction in training time, while maintaining high utility in noisy environments where DP-SGD 's performance degrades.
\end{itemize}

\section{Preliminaries}
\subsection{Differential-privacy}
We define a dataset $\mathcal{D}\triangleq \{\mathbf{x}_1,\ldots,\mathbf{x}_m\}\in \mathbb{R}^{m\times n}$, $\mathbf{x}_j\in \mathbb{R}^n$, that consists of a collection of $m$ records (rows) each representing a user's data with $n$ features. Also let $\mathbbm{1}\left(\cdot \right)$ be an indicator function which equals $1$ if the condition in $(\cdot)$ is true, and is equal to $0$ otherwise. We call $\mathcal{D}'=\{\mathbf{x}_1',\ldots,\mathbf{x}_m'\}\in \mathbb{R}^{m\times n}$ a neighboring dataset of $\mathcal{D}$ if it satisfies
\vspace{-0.25cm}
\begin{align*}
[\mathcal{D}-\mathcal{D}'] &\triangleq \sum_{j=1}^m \mathbbm{1}\left(\mathbf{x}_j\neq \mathbf{x}_j'\right)=1,
\end{align*}
\emph{i.e.,} it differs from $\mathcal{D}$ in only one user's data. 

\begin{definition}[$\epsilon$ DP \cite{dwork1}]
\label{def:DP}
A randomizing mechanism $\mathcal{M}(\mathcal{D})$ with domain $\mathbb{R}^{m\times n}$ satisfies $\epsilon$-DP if for any two neighboring datasets $\mathcal{D}$ and $\mathcal{D}'$, and for all randomizer outputs $\mathcal{S}\subseteq \mathrm{Range}(\mathcal{M})$,
\[
 \Pr(\mathcal{M}(\mathcal{D})\in \mathcal{S})\leq e^{\epsilon}\Pr(\mathcal{M}(\mathcal{D}')\in \mathcal{S}). 
\]
\end{definition}
\vspace{-0.25cm}

In this work we study both the SPOF and DP-SGD randomizing mechanisms. When environmental noise $\mathbf{n}_j\in \mathbb{R}^n$ is taken into consideration, the input to the $j$-th encoder of a DA is denoted by $\mathbf{x}_j+\mathbf{n}_j$, and in this setting, the goal of DP is to protect the {privacy of the, $j$-th user's data expressed as a} noiseless input $\mathbf{x}_j$. We demonstrate that the SPOF can leverage $\mathbf{n}_j$ to privatize the $j$-th user's data. On the other hand, DP-SGD is {less influenced by} such noise.


\begin{definition}[$\ell_1$ sensitivity \cite{dwork1}] The $\ell_1$ sensitivity $\Delta f \in \mathbb{R}$ of a function $f: \mathbb{R}^{m\times n}\to \mathbb{R}^n$ is  
\begin{equation}
	\Delta f \triangleq \max_{\substack{\mathcal{D}, \mathcal{D}'\in \mathbb{R}^{m\times n} \\ {[\mathcal{D}-\mathcal{D}']=1} }} \lVert {f(\mathcal{D})-f(\mathcal{D}') \lVert_1},
\end{equation}
\end{definition} 
\vspace{-0.25cm}
\noindent where {$\lVert \cdot \lVert_1$} denotes $\ell_1$-norm. The $\ell_1$ sensitivity measures the maximum amount of change in the output of $f(\mathcal{D})$ when it is applied to two neighboring datasets. Higher sensitivity means that the function’s output could vary more due to the inclusion or exclusion of a single user's data. Thus, a function with higher sensitivity requires more noise to be added to its output to obscure the influence of a user's data. 

The PDF of a random variable (r.v.) $X$, with realization $x$, drawn from the Laplace distribution with zero mean and variance $2d^2$ is given by {$\mathrm{Lap}(d): x \mapsto \frac{1}{2d} e^{-\frac{|x|}{d}}$.} Given a randomizing mechanism $\mathcal{M}(\mathcal{D},f(\mathcal{D}),\epsilon)=f(\mathcal{D})+(N_1,\ldots,N_n)$, as long as $N_i\sim \mathrm{Lap}(\Delta f/\epsilon)~ \forall i \in\{1,\ldots,n\}$, the output of $f(\mathcal{D})$ 
satisfies $\epsilon$-DP \cite{dwork2}.

\begin{figure*}[h]
  \centering
  \includegraphics[scale=0.8]{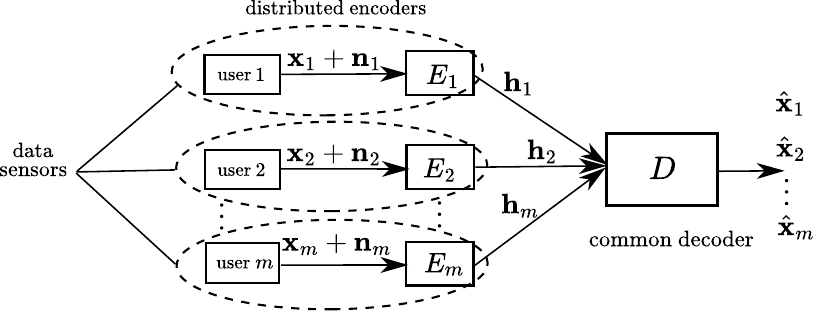}
  \vspace{-0.12cm}
  \caption{The proposed DA framework with noisy inputs $\mathbf{x}_1+\mathbf{n}_1,\ldots,\mathbf{x}_m+\mathbf{n}_m$. The noiseless case is obtained when $\mathbf{n}_j=\mathbf{0}$. The encoders compress user data generated by sensors (\emph{e.g.}, body sensor in a WBAN).}
\label{fig:setup1} 
\end{figure*}

\subsection{Autoencoder}

An autoencoder with parameters $\mathbf{W} {\in \mathbb{R}^{n\times l}}$ and $\hat{\mathbf{W}} {\in \mathbb{R}^{l\times n}}$, {where $l\leq n$,} consists of an encoding function $E_{\mathbf{W}}: \mathbf{x}\mapsto \mathbf{h}$, and a decoding function $D_{\hat{\mathbf{W}}}: \mathbf{h}\mapsto \hat{\mathbf{x}} \in\mathbb{R}^n$, which is used to reconstruct $\hat{\mathbf{x}}$ from {$\mathbf{h}$}. The performance of an autoencoder is evaluated using a distortion function $d(\mathbf{x}, \hat{\mathbf{x}})\in [0,\infty)$ which quantifies the difference between the original data $\mathbf{x}$ and its reconstruction $\hat{\mathbf{x}}$. A higher value of distortion indicates a larger difference between $\mathbf{x}$ and $\hat{\mathbf{x}}$, and vice versa.

Let $\mu$ denote the probability distribution from which a training sample $\mathbf{x}$ is drawn. The goal of autoencoding is to optimize the parameters $\mathbf{W},\hat{\mathbf{W}}$ such that the loss
\begin{equation}
	L(\mathbf{W},\hat{\mathbf{W}})\triangleq \mathbb{E}_{\mathbf{x}\sim \mu} \bigg[d(\mathbf{x},D_{\hat{\mathbf{W}}}(E_{\mathbf{W}}(\mathbf{x}))) \bigg]
	\label{eq:dist}
\end{equation} 
is minimized, which is typically accomplished using SGD. The distortion function is defined as $d(\mathbf{x},\hat{\mathbf{x}})\triangleq\lVert \mathbf{x}-\hat{\mathbf{x}}\lVert_2^2$, where $\lVert \cdot \lVert_2$ is the $\ell_2$-norm. 
Let the model parameters $\mathbf{W}$ and $\hat{\mathbf{W}}$ represent the weight matrices of the encoding and decoding layers, respectively. Consider the probability 
\begin{align}
p_i(\mathbf{x})&=\sigma(\hat{\mathbf{W}}_{(:i)}^\top \mathbf{h}) = \frac{1}{1+e^ {-\hat{\mathbf{W}}_{(:i)}^\top\mathbf{h}}},
\end{align}
where $\hat{\mathbf{W}}_{(:i)} \in \mathbb{R}^{l}$ is the $i$-th column of $\hat{\mathbf{W}}$, $\mathbf{h}=\sigma(\mathbf{W}^\top \mathbf{x})$ is a {$l$-dimensional} vector representing an encoding layer output, and $\sigma(\cdot)$ is a sigmoid activation function. In a SA, the distortion in (\ref{eq:dist}) can be minimized by minimizing a binary cross-entropy loss, $L_C\left(\mathbf{W},\hat{\mathbf{W}}, {\mathbf{x}}\right)$  \cite{Bengio06}, according to
\begin{align}
\min_{\mathbf{W},\hat{\mathbf{W}}} L_C\left(\mathbf{W},\hat{\mathbf{W}}, {\mathbf{x}}\right) 
\label{eq:opt2}
&\triangleq\min_{\mathbf{W},\hat{\mathbf{W}}} \sum_{i=1}^n \left(x_i \log \left(1+e^ {-\hat{\mathbf{W}}_{(:i)}^\top\mathbf{h}}\right) \right. \nonumber \\
&~~~ \left. + (1-x_i)\log \left(1+e^ {\hat{\mathbf{W}}_{(:i)}^\top\mathbf{h}}\right) \right).
\end{align}
In the next section, we demonstrate how this loss function can be extended and applied to a DA. We then express it in polynomial form using a Taylor expansion, the coefficients of which are perturbed via SPOF and DP-SGD.

\section{Derivation of DA Loss Function and Its Polynomial Approximation}
\label{sec:DA_loss}
Consider a DA comprising of $m$ encoders with inputs $\mathbf{x}_1,\ldots,\mathbf{x}_m$, and {$\mathbf{h}_j\triangleq\sigma\left(\mathbf{W}_{j}^\top \mathbf{x}_j\right) \in \mathbb{R}^l$, $l<n$, is the output of the encoder with input $\mathbf{x}_j\in \mathbb{R}^n$. The $i$-th element $0<h_{j,i}<1$ of $\mathbf{h}_j$ is the output of a sigmoid activation function $\sigma(\cdot)$}. The {DA} also comprises a single decoder with inputs $\mathbf{h}_1,\ldots,\mathbf{h}_m$ and {the} set of outputs $\hat{\mathbf{x}}_1,\ldots,\hat{\mathbf{x}}_m$, where $\hat{\mathbf{x}}_j\in \mathbb{R}^n$. A mini-batch $\mathcal{D}\triangleq\{\mathbf{x}_1,\ldots,\mathbf{x}_m\}$ of $m$ training samples represent the data of all the $m$ users, where the $i$-th feature $x_{j,i}$ of $\mathbf{x}_j$ satisfies $0\leq x_{j,i} \leq 1$. Also, let $\mathbf{W}_j {\in \mathbb{R}^{n\times l}}$ and $\hat{\mathbf{W}} {\in \mathbb{R}^{ml\times n}}$. A DA consists of {encoding functions $E_{\mathbf{W}_j}: \mathbf{x}_j\mapsto \mathbf{h}_j$ for all $j\in \{1,\ldots,m\}$}, and a decoding function $D_{\hat{\mathbf{W}}}: \{\mathbf{h}_1,\ldots,\mathbf{h}_m\}\mapsto \{\hat{\mathbf{x}}_1,\ldots,\hat{\mathbf{x}}_m\}$ to reconstruct the encoded signals. Even though the decoder processes all $m$ inputs simultaneously, its outputs are generated sequentially over $m$ iterations. {We choose this sequential design to reduce the size of $\hat{\mathbf{W}}$, which in turn reduces DA training complexity}. Let $L_C(\mathbf{W}_{1},\ldots, \mathbf{W}_{m},\hat{\mathbf{W}}, {\mathbf{x}})$ be the loss function of a DA with parameters $\mathbf{W}_{1},\ldots, \mathbf{W}_{m},\hat{\mathbf{W}}$, which are jointly optimized according to
\begin{align}
&\min_{\mathbf{W}_{1},\ldots,\mathbf{W}_{m},\hat{\mathbf{W}}} L_C\left(\mathbf{W}_{1},\ldots,\mathbf{W}_{m},\hat{\mathbf{W}}, {\mathbf{x}}\right),
\end{align}
rendering the optimization significantly more complex than the one in (\ref{eq:opt2}), particularly if $m$ is large. A general DA framework is depicted in Fig.~\ref{fig:setup1}, in which the weight matrix of the $j$-th encoder $E_j$ and decoder $D$ is denoted by $\mathbf{W}_j$ and $\hat{\mathbf{W}}$, respectively.

{Let $\hat{\mathbf{W}}_{(:i[j])} \in \mathbb{R}^l$ be the $j$-th sub-column, in the $i$-th column of $\hat{\mathbf{W}}$ consisting of rows with indices $(j-1)l+1,(j-1)l+2,\ldots,jl$.} Also let $z_{j,i}\triangleq \hat{\mathbf{W}}_{(:i[j])}^\top\mathbf{h}_j$. Let $f_{j,i,1}(z)\triangleq x_{j,i} \log (1+e^{-z})$ and $f_{j,i,2}(z)\triangleq (1-x_{j,i}) \log (1+e^z)$. Then, the DA loss function is given by
\begin{align}
&L_C\left(\mathbf{W}_{1},\ldots, \mathbf{W}_{m},\hat{\mathbf{W}}, {\mathbf{x}}\right) \nonumber \\
&\triangleq \sum_{j=1}^m \sum_{i=1}^n \bigg\{ \left(x_{j,i} \log \left(1+e^{-\hat{\mathbf{W}}_{(:i[j])}^\top\mathbf{h}_j}\right)\right. \nonumber \\
&~~~ \left. + (1-x_{j,i})\log \left(1+e^{\hat{\mathbf{W}}_{(:i[j])}^\top\mathbf{h}_j}\right) \right) \nonumber \\
=& \sum_{j=1}^m \sum_{i=1}^n f_{j,i,1}\left(z_{j,i}\right)+f_{j,i,2}\left(z_{j,i}\right) = \sum_{j=1}^m \sum_{i=1}^n \sum_{k=1}^2 f_{j,i,k}\left(z_{j,i}\right).
\label{eq:lwww}
\end{align}
Since this function is differentiable{, its Taylor expansion is expressed as}
\begin{align}
&L_C\left(\mathbf{W}_{1},\ldots, \mathbf{W}_{m},\hat{\mathbf{W}}, {\mathbf{x}}\right)= \sum_{j=1}^m L_j \nonumber \\
& = \sum_{j=1}^m \sum_{i=1}^n \sum_{k=1}^2 \sum_{r=0}^\infty \frac{f_{j,i,k}^{(r)}(a)}{r!} \left(z_{j,i}-a \right)^r, 
\label{eq:loss_orig}
\end{align}
where $f_{j,i,k}^{(r)}(a)$ is the $r$-th derivative of $f_{j,i,k}(z_{j,i})$ evaluated at point $a$.
 
Note that in \eqref{eq:loss_orig}, \( z_{j,i} \) is an independent variable of the function \( f_{j,i,k}(z_{j,i}) \), and the Taylor expansion is performed with respect to this variable. Therefore, \( z_{j,i} \) must remain unperturbed, as perturbing it would invalidate the series expansion, whose coefficients are derived in the next section under the assumption that \( z_{j,i} \) is a fixed input.

\section{SPOF in Presence of Noiseless Inputs}
\label{sec:DA_noiseless}
For facilitating DP in DA systems, a DP randomizing mechanism {with low perturbation complexity} like SPOF is desirable to alleviate the high {complexity of optimizing} the $m+1$ model parameters. In this section, we derive the sensitivity of our SPOF approach and discuss how DA user privacy is achieved using it. 

Consider the DA framework in Fig.~\ref{fig:setup1} where noiseless inputs $\mathbf{x}_1,\ldots,\mathbf{x}_m$ are transmitted instead of noisy inputs $\mathbf{x}_1+\mathbf{n}_1,\ldots,\mathbf{x}_m+\mathbf{n}_m$, with environmental noise $\mathbf{n}_j\in \mathbb{R}^n$. We first derive an approximate DA loss function, the coefficients of which are perturbed using SPOF to provide DP {in the absence of environmental noise}. Considering a second order Taylor series, an approximation of (\ref{eq:loss_orig}) is given as
\begin{align}
\hat{L}_j&= \sum_{i=1}^n \sum_{k=1}^2 \left( f_{j,i,k}^{(0)}(a) + f_{j,i,k}^{(1)}(a) (z_{j,i}-a )+ \right. \nonumber \\
&~~~ \left. \frac{f_{j,i,k}^{(2)}(a)}{2} (z_{j,i}-a )^2 \right).
  \label{eq:ljhat}
\end{align}
For $k=1$ we have
\begin{align*}
	f_{j,i,1}^{(0)}(a)&=  x_{j,i} \log (1+e^{-a}), \nonumber \\
	 f_{j,i,1}^{(1)}(a)&= -\frac{ x_{j,i}}{1+e^{a}},~ \text{ and }
	 f_{j,i,1}^{(2)}(a)=  x_{j,i} \frac{e^{a}}{(1+e^{a})^2}.  
\end{align*}
Similarly, for $k=2$, 
\begin{align*}
	f_{j,i,2}^{(0)}(a)&= (1-x_{j,i}) \log (1+e^{a}), \nonumber \\
f_{j,i,2}^{(1)}(a)&= \frac{ (1-x_{j,i})}{1+e^{-a}}, \text{ and } \nonumber \\
f_{j,i,2}^{(2)}(a)&= - (1-x_{j,i}) \frac{e^{-a}}{(1+e^{-a})^2}. 
\end{align*}
Thus the loss function for the $j$-th user can be rewritten as
\begin{align}
\hat{L}_j&= \sum_{i=1}^n \left(x_{j,i} \log (1+e^{-a}) -\frac{ x_{j,i}}{1+e^{a}}  (z_{j,i}-a ) \right. \nonumber \\
&~~~ \left. + x_{j,i}  \frac{e^{a}}{(1+e^{a})^2}  (z_{j,i}-a )^2+ (1-x_{j,i}) \log (1+e^{a})\right. \nonumber \\
  &\left. +\frac{ (1-x_{j,i})}{1+e^{-a}} (z_{j,i}-a)-  (1-x_{j,i}) \frac{e^{-a}}{(1+e^{-a})^2} (z_{j,i}-a)^2 \right).
\label{eq:loss_jl_hat}
\end{align}

\begin{prop}
The error, $E_j$, caused by second order Taylor approximation of the $j$-th user's loss function is bounded as
\begin{equation}
	E_j \leq 2G \left( e^\delta - 1 - \delta - \frac{\delta^2}{2} \right),
	\label{eq:error}
\end{equation}
where $G, \delta \in \mathbb{R}^+$.
\end{prop}
\begin{proof}
	See Appendix~\ref{app:err_funcApprox}.
\end{proof}
As a result, $E_j\to 0 $ as \( \delta \to 0 \), implying that the error is negligible if $z_{j,i}$ is close to $a$.

The $j$-th loss function can be simplified as 
\begin{align}
\hat{L}_j&= \sum_{i=1}^n  \left( \log \left(\left(\frac{1+e^{-a}}{1+e^a}\right)^{x_{j,i}}(1+e^a)\right)  \right. \nonumber \\
&~~~ \left. + \left(\frac{1}{1+e^{-a}}-x_{j,i}\right)  (z_{j,i}-a ) \right. \nonumber \\
&~~~ \left. + \left(\frac{2x_{j,i}-1}{2+e^a+e^{-a}} \right)  (z_{j,i}-a )^2 \right). 
\label{eq:loss_jl_hat_simp}
\end{align}
If $a=0$, the resulting loss function is expressed as 
\vspace{-0.15cm}
\begin{align}
\tilde{L}_j&= n\log 2 + \sum_{i=1}^n  \left(\left(0.5-x_{j,i}\right)z_{j,i} + \left(0.5x_{j,i}-0.25 \right)z_{j,i}^2 \right) \nonumber \\
&= n\alpha_{j,i,1} + \sum_{i=1}^n  \left(\alpha_{j,i,2} z_{j,i} + \alpha_{j,i,3} z_{j,i}^2 \right), 
  \label{eq:lossSimp} 
\end{align}
with the notations $\alpha_{j,i,1}\triangleq \log 2$, $\alpha_{j,i,2}\triangleq \left(0.5-x_{j,i}\right)$, and $\alpha_{j,i,3}\triangleq \left(0.5 x_{j,i}-0.25 \right)$. Based on these coefficients, we derive the SPOF sensitivity upper-bound $\bar{\Delta}_{\mathrm{SPOF}}$ in Appendix~\ref{app:SPOF_sens}.

\begin{prop}
	Let $p_{\mathcal{D}}(\mathbf{q})$ be the PDF of a randomizing mechanism $\mathcal{M}(\mathcal{D},f(\mathcal{D}),\epsilon)$ computed at an arbitrary point $\mathbf{q}\in \mathbb{R}^n$, and let $p_{\mathcal{D}'}(\mathbf{q})$ be the PDF of $\mathcal{M}(\mathcal{D}',f(\mathcal{D}'),\epsilon)$ computed at the same point. The SPOF randomization scheme for a differentially-private DA (DP-DA) is expressed as \begin{align*}
	\mathcal{M}(\mathcal{D},f(\mathcal{D}),\epsilon)\triangleq  f(\mathcal{D}) +(N_1,\ldots,N_n),
\end{align*}
where the $i$-th component of the query function $f(\mathcal{D})_i=\sum_{p=2}^3 \alpha_{j,i,p}$, $f(\mathcal{D}')_i=\sum_{p=2}^3 \alpha_{j,i,p}'$, and $N_i\sim \mathrm{Lap}(\bar{\Delta}_{\mathrm{SPOF}}/\epsilon)$. For noise scale $d=\bar{\Delta}_{\mathrm{SPOF}}/\epsilon$ we obtain
\[\frac{p_{\mathcal{D}}(\mathbf{q})}{p_{\mathcal{D}'}(\mathbf{q})} \leq e^\epsilon.\]
\label{prop:spof}
\end{prop}
\begin{proof}
	See Appendix~\ref{app:nl_SPOF}.
\end{proof}
\noindent This result implies that SPOF satisfies $\epsilon$-DP, where $\epsilon$ is the privacy budget.

\section{Loss Stabilization {via DA Weight Modification}}
\label{sec:DA_lossStab}
Now, we discuss a key component of SPOF that adds a constant vector $\mathbf{c}_j$, that introduces a controllable bias into the decoded output, to the decoder weights for loss stabilization affects the DP guarantee provided by SPOF. 

\begin{prop}
By using the sensitivity 
	\[\hat{\Delta}_{\mathrm{SPOF}}= \bar{\Delta}_{\mathrm{SPOF}}+ n c_j\left(\frac{c_j}{2}+2 \right),\]
\end{prop}
\noindent for calibrating DP noise, SPOF preserves the $\epsilon$-DP guarantee due to addition of the loss constant  $\mathbf{c}_j$, where $\bar{\Delta}_{\mathrm{SPOF}}$ is SPOF sensitivity without loss stabilization constant, 
and $c_j$ is the $j$-th element of  $\mathbf{c}_j$.
\begin{proof}
	See Appendix~\ref{app:loss_const}.
\end{proof}
As a result, this modification can achieve the same DP guarantee as the original SPOF by increasing the sensitivity from $\bar{\Delta}_{\mathrm{SPOF}}$ to $\hat{\Delta}_{\mathrm{SPOF}}$. This proof can be extended in a similar fashion to noisy inputs by replacing $\alpha_{j,i,k}$ and $\bar{\Delta}_{\mathrm{SPOF}}$ with their ``noisy'' counterparts $\alpha_{j,i,k}^{[N]}$ and $\bar{\Delta}_{\mathrm{SPOF}}^{[N]}$, respectively.

The application of SPOF to privatize user data in a DA is outlined in Algorithm 1. Assuming that an adversary has access to the DA learnable parameters through the query function $f(\mathcal{D})=\sum_{p=2}^3 \alpha_{j,i,p}$, our goal is to preserve the privacy of the $j$-th user's data by perturbing the coefficients of the loss function in (\ref{eq:lossSimp}). This is accomplished in {Step 5} of Algorithm 1. Let $M_i$ represent a r.v. sampled from a zero mean Laplace distribution with variance $\left(\frac{\hat{\Delta}_{\mathrm{SPOF}}}{\epsilon}\right)^2$. In SPOF, the $i$-th component of $f(\mathcal{D})$ is privatized according to
\begin{align*}
	f(\mathcal{D})_i+M_i &= \alpha_{j,i,2}+\alpha_{j,i,3}+M_i \nonumber \\
	&= \alpha_{j,i,2}+\alpha_{j,i,3}+(M_{i,1}+M_{i,2}) \nonumber \\
	&= \alpha_{j,i,2}+M_{i,1}+\alpha_{j,i,3}+M_{i,2},
\end{align*}
where $M_{i,1}$ and $M_{i,2}$ are independent and identically distributed (i.i.d.) zero mean Laplace r.v.s, each with variance equal to half that of $M_i$. 

Thus, {Step 5} perturbs $\alpha_{j,i,p}$ for $p\in \{2,3\}$ with noise sampled from $\mathrm{Lap}\left(\frac{\hat{\Delta}_{\mathrm{SPOF}}}{\sqrt{2}\epsilon}\right)$, where $\hat{\Delta}_{\mathrm{SPOF}}=\bar{\Delta}_{\mathrm{SPOF}}+ n c_j\left(\frac{c_j}{2}+2 \right)$, $c_j=\mathbf{c}_j^\top \mathbf{h}_j$, and $\bar{\Delta}_{\mathrm{SPOF}}$ is chosen according to (\ref{eq:fm_bound}) of Appendix~\ref{app:SPOF_sens}.

In {Step 3}, we apply the loss stabilization constant vector $\mathbf{c}_j$, whose entries $c_{j,1}, \ldots, c_{j,l}$ all have the same value. This value is selected by sweeping over a range of positive values and choosing the value beyond which the DP-DA's reconstruction accuracy no longer changes for a fixed $\epsilon$. Adding this constant vector modifies the loss function as shown in Appendix \ref{app:loss_const}, making the loss function less sensitive to fluctuations in DP noise, improving the quality of the gradient updates in {Step 6}.

As described in {Step 5} of Algorithm 1, noise is added to each user's loss function coefficients, {enabling us to apply parallel composition of DP \cite{parallelDP} as the users hold disjoint private data.} Under this model, each user's data is privatized independently, and each retains the same $\epsilon$-DP guarantee. Therefore, privacy loss does not accumulate across users. 

Since the model parameters are learned solely by minimizing the perturbed loss function, SPOF enables DP with respect to the user’s data, even though the model parameters themselves are not directly perturbed.

Although gradient norms are not computed in SPOF, {Step~6} of Algorithm 1 can be adapted to BK \cite{small_cost} by using forward and backward hooks to obtain gradients more efficiently than standard backpropagation (BP), thereby reducing BP overhead. Since SPOF does not involve clipping, it does not benefit from GC. In contrast, BK in DP-SGD requires per-sample norm estimation and clipping \cite{small_cost, group_wise}.

{\linespread{0.9}\selectfont
\begin{algorithm}[h]
\caption{Training DP-DA Using SPOF}
\textbf{Input:} user dataset $\mathcal{D} = \{\mathbf{x}_j\}_{j=1}^m$, learning rate $\eta$, privacy parameter $\epsilon$\tcp*{}
  \For{each user $j \in \{1, \ldots, m\}$}{
    $\tilde{\mathbf{W}}_{(:i[j])}^\top \leftarrow \hat{\mathbf{W}}_{(:i[j])}^\top+ \mathbf{c}_j^\top~\forall i$ \tcp*{\comment{adding loss stabilization constant vector $\mathbf{c}_j$}}
    \For{each $i \in \{1,\ldots,n\}$}{
      $\alpha_{j,i,p} \leftarrow \alpha_{j,i,p} + \mathrm{Lap}\left(\frac{\hat{\Delta}_{\mathrm{SPOF}}}{\sqrt{2}\epsilon}\right) ~\forall p \in \{2,3\}$\tcp*{\comment{perturb coefficients of loss function $\tilde{L}_j$}
    }}
    compute gradient $\nabla \tilde{L}_j$ of $\tilde{L}_j$ via BK\tcp*{\comment{tracks all encoder and decoder weights}}
    update parameters via SGD: $(\mathbf{W}_j, \hat{\mathbf{W}}) \leftarrow (\mathbf{W}_j, \hat{\mathbf{W}}) - \eta \cdot \nabla \tilde{L}_j$ \\
  }
\label{alg:SPOF}
\end{algorithm}
}

Note that SPOF enforces privacy at the user level rather than the sample level. As a result, SPOF associates a separate loss function $\tilde{L}_j$ with each user $j$, and the model parameters (e.g., $\alpha_{j,i,p}$) explicitly depend on both the user index $j$ and the feature index $i$. This design enables user-level LDP. Notably, if we set $m=1$ and $\mathbf{c}_j=\mathbf{0}$, SPOF reduces to a single-user variant equivalent to FM, {in which $\alpha_{1,i,p}$ is updated for each new training sample in the dataset.}

{In Section \ref{sec:simulation}, we show that SPOF is advantageous over standard DP mechanisms such as DP-SGD that adds noise directly to gradients. }

\section{SPOF in Presence of Noisy Inputs} 
\label{sec:DA_noisy_loss}
Let $\mathbf{n}_j$ denote environmental noise {shown in Fig.~\ref{fig:setup1}}, affecting the $j$-th user's data, consisting of i.i.d. noise samples from a standard normal distribution $\mathcal{N}(0,\sigma^2)$. 

\begin{thm}
When $\mathbf{n}_j\neq \mathbf{0}$, the coefficients of the loss function in (\ref{eq:lossSimp}) are modified, resulting in the loss function 
\begin{align}
\tilde{L}_j^{[N]}&= n\log 2 + \sum_{i=1}^n  \left(b_j\left(0.5-x_{j,i}\right)  \left(z_{j,i}-\frac{t_{j,i}}{b_j} \right) \right. \nonumber \\
&~~~ \left. + b_j\left(0.5x_{j,i}-0.25 \right)  \left(z_{j,i}-\frac{t_{j,i}}{b_j} \right)^2 \right) \nonumber \\
&= n\alpha_{j,i,1}^{[N]} + \sum_{i=1}^n  \left(\alpha_{j,i,2}^{[N]} \left(z_{j,i}-\frac{t_{j,i}}{b_j} \right) \right. \nonumber \\
&~~~ \left. + \alpha_{j,i,3}^{[N]} \left(z_{j,i}-\frac{t_{j,i}}{b_j} \right)^2 \right),
\label{eq:lossLjN}
\end{align}
with coefficients $\alpha_{j,i,1}^{[N]}= \alpha_{j,i,1}$, $\alpha_{j,i,2}^{[N]}= b_j \alpha_{j,i,2}$, and $\alpha_{j,i,3}^{[N]}=b_j \alpha_{j,i,3}$, where $b_j, t_{j,i} \in \mathbb{R}$. 
\end{thm}
\begin{proof}
	See Appendix~\ref{app:noisy_inp_loss}.
\end{proof}

If environmental noise (\emph{e.g.}, body senor noise in case of WBAN) is present, we denote the SPOF sensitivity as $\bar{\Delta}_{\mathrm{SPOF}}^{[N]}$ instead of $\bar{\Delta}_{\mathrm{SPOF}}$, and compute the difference between its $i$-th {term in (\ref{eq:delFMorig0_2})} and that of $\bar{\Delta}_{\mathrm{SPOF}}^{[N]}$ as a function of $b_j$ in (\ref{eq:fmdiff1}) of Appendix~\ref{app:sensor_noise}. Moreover, in Appendix \ref{app:comb}, we show that $\Pr[b_\max< 1]\geq 0.5$ for sufficiently large environmental noise parameter $\sigma$, where $b_j\leq b_\max$ and $j\in\{1,\ldots,m\}$.

The amount of DP noise in SPOF can be reduced by replacing $\bar{\Delta}_{\mathrm{SPOF}}$ with $\bar{\Delta}_{\mathrm{SPOF}}^{[N]}=b_j \bar{\Delta}_{\mathrm{SPOF}}$ to achieve the same privacy level $\epsilon$ as in the noiseless case, \emph{i.e.,} when $\mathbf{n}_j= \mathbf{0}$ (see Appendix~\ref{app:n_SPOF}). To take into account loss stabilization constant, we apply the sensitivity $\hat{\Delta}_{\mathrm{SPOF}}^{[N]}= \bar{\Delta}_{\mathrm{SPOF}}^{[N]}+ n c_j\left(\frac{c_j}{2}+2 \right)$ in {Step 5} of Algorithm 1 when $\mathbf{n}_j\neq \mathbf{0}$. It is also necessary to replace $\tilde{L}_j$, $\alpha_{j,i,2}$, and $\alpha_{j,i,3}$ in Algorithm 1 with their ``noisy'' counterparts $\tilde{L}_j^{[N]}$, $\alpha_{j,i,2}^{[N]}$, and $\alpha_{j,i,3}^{[N]}$, respectively, which are derived in Appendix~\ref{app:noisy_inp_loss}. This leads to the following result.
{
\begin{prop}
 As $\Pr[b_\max< 1]\geq 0.5$, it follows that $\Pr[b_j< 1]\geq 0.5$. Moreover, since $\bar{\Delta}_{\mathrm{SPOF}}^{[N]}=b_j \bar{\Delta}_{\mathrm{SPOF}}$, we obtain $\Pr[\hat{\Delta}_{\mathrm{SPOF}}^{[N]}<\hat{\Delta}_{\mathrm{SPOF}}]=\Pr[b_j< 1]\geq 0.5$.
 \label{prop:prob}
\end{prop}
}
\begin{rem}
Since $\hat{\Delta}_{\mathrm{SPOF}}^{[N]}$ is used in {Step 5} of Algorithm 1 instead of $\hat{\Delta}_{\mathrm{SPOF}}$ in the presence of environmental noise, Proposition \ref{prop:prob} implies that under noisy conditions, SPOF or its variant, FM, can achieve a higher reconstruction accuracy with respect to the noiseless input condition, due to reduced noise requirements for achieving $\epsilon$-DP guarantee.
\label{rem:spof_guarantee}
\end{rem}

\section{Application of DP-SGD to a DA}
\label{sec:loss_dpsgd}
The $i$-th component of the gradient, $\nabla \hat{L}_{j}$, used by an SGD algorithm for updating model parameters is the first-order derivative of $L_j$ with respect to the $i$-th model parameter $z_{j,i}$. In the case of DA, the gradient processed by SGD for a training sample corresponding to the $j$-th user is given by
\begin{align}
\nabla \hat{L}_{j} &\triangleq \left[f_{j,1,1}^{(1)}(a)+f_{j,1,2}^{(1)}(a),\ldots,f_{j,n,1}^{(1)}(a)+f_{j,n,2}^{(1)}(a) \right]{\in \mathbb{R}^n},
\label{eq:grad}
\end{align}
where the $i$-th component
\begin{align*}
	\nabla \hat{L}_{j,i}&\triangleq f_{j,i,1}^{(1)}(a)+f_{j,i,2}^{(1)}(a)= \frac{-x_{j,i}}{1+e^a} +\frac{1-x_{j,i}}{1+e^{-a}}  \nonumber \\
&~~~ = \frac{e^a}{1+e^a}-x_{j,i}
\end{align*}
of $\nabla \hat{L}_{j}$ is the first order derivative of $\hat{L}_j$ with respect to $z_{j,i}$. 
On the other hand, when environmental noise $\mathbf{n}_j\neq \mathbf{0}$, the gradient is given by
\begin{align*}
	\tilde{\nabla} \hat{L}_{j,i}&\triangleq \tilde{f}_{j,i,1}^{(1)}(a)+\tilde{f}_{j,i,2}^{(1)}(a) \nonumber \\
&~~~ = b_j f_{j,i,1}^{(1)}(a)+ b_j f_{j,i,2}^{(1)}(a) = b_j \nabla \hat{L}_{j,i}.
\end{align*}
In the remainder of the paper, we will use $\nabla \bar{L}_{j,i}$ to denote a clipped gradient, \emph{i.e.,}
\begin{align}
	\nabla \bar{L}_{j,i}\triangleq \frac{\frac{e^a}{1+e^a}-x_{j,i}}{\max\left(1, \frac{\lVert \nabla \hat{L}_j \lVert_2}{C}\right)},
	\label{eq:clippedgrad}
\end{align}
where the denominator in (\ref{eq:clippedgrad}) is chosen to match the gradient clipping technique used in \cite{dpsgd1} {with clipping threshold $C$.} Based on the gradient, we derive the DP-SGD sensitivity upper-bound $\Delta_{\mathrm{SGD}}$ and $\Delta_{\mathrm{SGD}}^{[N]}$ in Appendix~\ref{app: dp_sgd_sens}, which is necessary for calibrating the DP noise scale in noiseless and noisy input conditions, respectively. 

Note that using the proof technique of Proposition \ref{prop:spof}, one can also show that DP-SGD satisfies $\epsilon$-DP using Laplace noise, where $f(\mathcal{D})_i=\nabla \bar{L}_{j,i}$, and $\bar{\Delta}_{\mathrm{SPOF}}$ is replaced by $\Delta_{\mathrm{SGD}}$.

We propose Algorithm 2 in which DP-SGD is applied to a DA. In {Step 6}, $\Delta_{\mathrm{SGD}}$ is replaced by $\Delta_{\mathrm{SGD}}^{[N]}$ when $\mathbf{n}_j\neq \mathbf{0}$. {Step~3} of Algorithm~2 uses BK to record activations and output gradients via hooks, and use these to compute the norm in {Step~4} via equation~(2) of \cite{small_cost}. GC is applied here as described in \cite{group_wise}. While these enhancements improve BP efficiency, they do not change the overall arithmetic complexity of gradient norm and clipping operations analyzed in Appendix~\ref{app:complexity}. 

By comparing (\ref{eq:dpsgd_noise}) with the noiseless case in (\ref{eq:dpsgd_noisl}), we notice that $\Delta_{\mathrm{SGD}}^{[N]}= b_j \Delta_{\mathrm{SGD}}$ if $\lVert\nabla \hat{L}_j \lVert_2< C$. {Since $\Pr[b_\max< 1]\geq 0.5$, it follows that $\Pr[b_j< 1]\geq 0.5$. As a result, under noisy conditions, we obtain $\Pr\left[\Delta_{\mathrm{SGD}}^{[N]}<\Delta_{\mathrm{SGD}}\right]\geq 0.5$, provided that $\lVert\nabla \hat{L}_j \lVert_2<C$. Consequently, whether or not DP noise level in {Step 6} of Algorithm 2 can be reduced depends on $\nabla \hat{L}_j$. On the other hand, according to Proposition \ref{prop:prob}, $\Pr\left[\bar{\Delta}_{\mathrm{SPOF}}^{[N]}<\bar{\Delta}_{\mathrm{SPOF}}\right]\geq 0.5$ irrespective of $\nabla \hat{L}_j$, implying that SPOF is more likely to leverage environmental noise for privacy than DP-SGD.}

Note that setting \( a = 0 \) in (\ref{eq:lossSimp}) renders \( \alpha_{j,i,1} \) as a constant independent of \( x_{j,i} \), so it need not be perturbed, reducing the perturbation complexity in SPOF. In {Fig.~\ref{fig:sensitivity1}} of Appendix~\ref{app:sensitivity_optimize}, it is shown that choosing $a=0$ helps reduce the $i$-th sensitivity {term} of SPOF to almost its minimum value. The same appendix also compares the sensitivity difference between SPOF and DP-SGD.


{\linespread{0.9}\selectfont  
\begin{algorithm}[h]
   \caption{Training DP-DA Using DP-SGD}
   {\bfseries Input:} user dataset $\mathcal{D} = \{\mathbf{x}_j\}_{j=1}^m$, learning rate $\eta$, clipping norm $C$, privacy parameter $\epsilon$\\
		\For{each user $j \in \{1, \ldots, m\}$} {
   			compute gradient $\nabla \hat{L}_j$ of $\hat{L}_j$ via BK \tcp*{\comment{tracks all encoder and decoder weights}}
			apply GC to clip gradient $\nabla \bar{L}_{j}\leftarrow \frac{\nabla \hat{L}_j}{\max\left(1, \frac{\lVert \nabla \hat{L}_j \lVert_2}{C}\right)}$\tcp{} 
			\For{each $i\in\{1,\ldots,n\}$} {
			 	$\nabla \bar{L}_{j,i}\leftarrow \nabla \bar{L}_{j,i}+\mathrm{Lap}(\Delta_{\mathrm{SGD}}/\epsilon)$\tcp*{\comment{perturb clipped gradient}}
			 } 
			}
   			update parameters via SGD: $(\mathbf{W}_j, \hat{\mathbf{W}}) \leftarrow (\mathbf{W}_j, \hat{\mathbf{W}}) - \eta \cdot \nabla \bar{L}_j$ \\
	\label{alg:DPSGD}
\end{algorithm}
}

\section{Application of SPOF to a DP-WBAN}
\label{sec:simulation}
This section applies the proposed SPOF approach to facilitate multi-user privacy in a WBAN, resulting in a DP-WBAN, wherein each user’s body sensor acts as an encoder. In such a setting (see Fig.~\ref{fig:setup1}), $\mathbf{x}_j$ could represent the $j$-th patient's physiological data such as the amount of exercise done per week, heart rate, blood pressure, and so on. This data is transmitted to a hospital where the compressed data $\mathbf{h}_1,\ldots,\mathbf{h}_m$ from $m$ independent patients are decoded by $D$ for health monitoring. From a WBAN perspective, Definition~\ref{def:DP} implies that an adversary at the hospital cannot distinguish between neighboring datasets $\mathcal{D}$ and $\mathcal{D}'$ by observing the outputs of $\mathcal{M}(\mathcal{D})$ and $\mathcal{M}(\mathcal{D}')$. Thus, the adversary cannot determine whether the data of the $j$-th patient was included in $\mathcal{D}$, thus preserving the patient's privacy.

The Fitbit fitness tracking dataset \cite{fitbit} used in \cite{wban2} is chosen in this work for training the DA. The dataset includes daily patient activity metrics such as total distance, total steps, calories burned, and eleven {other} parameters which are all positive. A dataset $\mathcal{D}$ can be viewed as a $m\times n$ matrix, where each column is normalized by the maximum column value in that column of the original data. This normalization ensures that the input data satisfies the assumption $0\leq x_{j,i} \leq 1$. Each row consists of data for a single user only. We consider $n=14$, $l=7$, $|\mathcal{D}|=m$, and $36480$ training samples for $m=2$. Additionally, in case of DP-SGD, we chose norm bound $C=4$. 


\begin{figure}[h]
  \centering
  \begin{minipage}[t]{0.47\textwidth}
    \centering
    \includegraphics[scale=0.42]{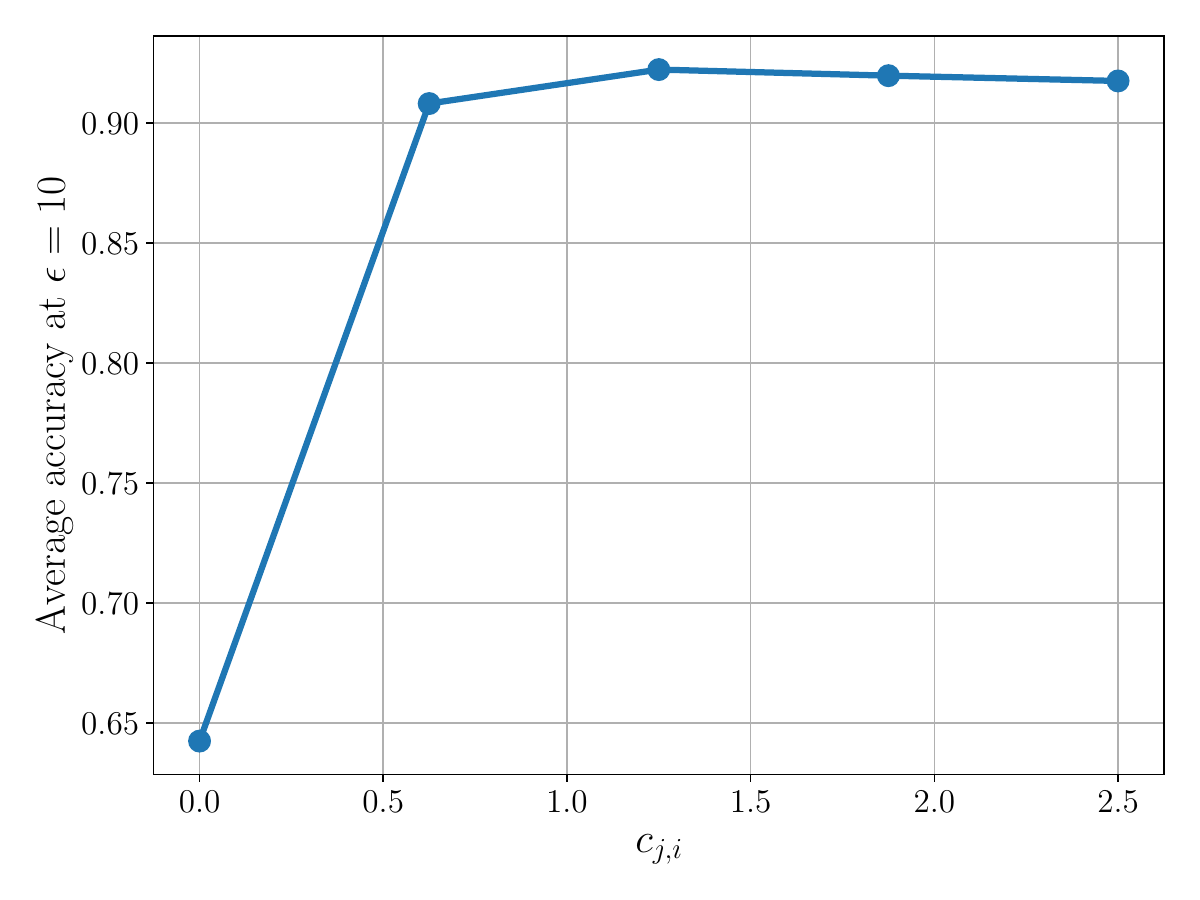}
    \vspace{-0.2cm}
    \caption{Effect of loss stabilization constant $c_{j,i}$ on accuracy of our DP-WBAN using SPOF.}
    \label{fig:const_sweep}
  \end{minipage}
  \hfill
  \begin{minipage}[t]{0.47\textwidth}
    \centering
    \includegraphics[scale=0.47]{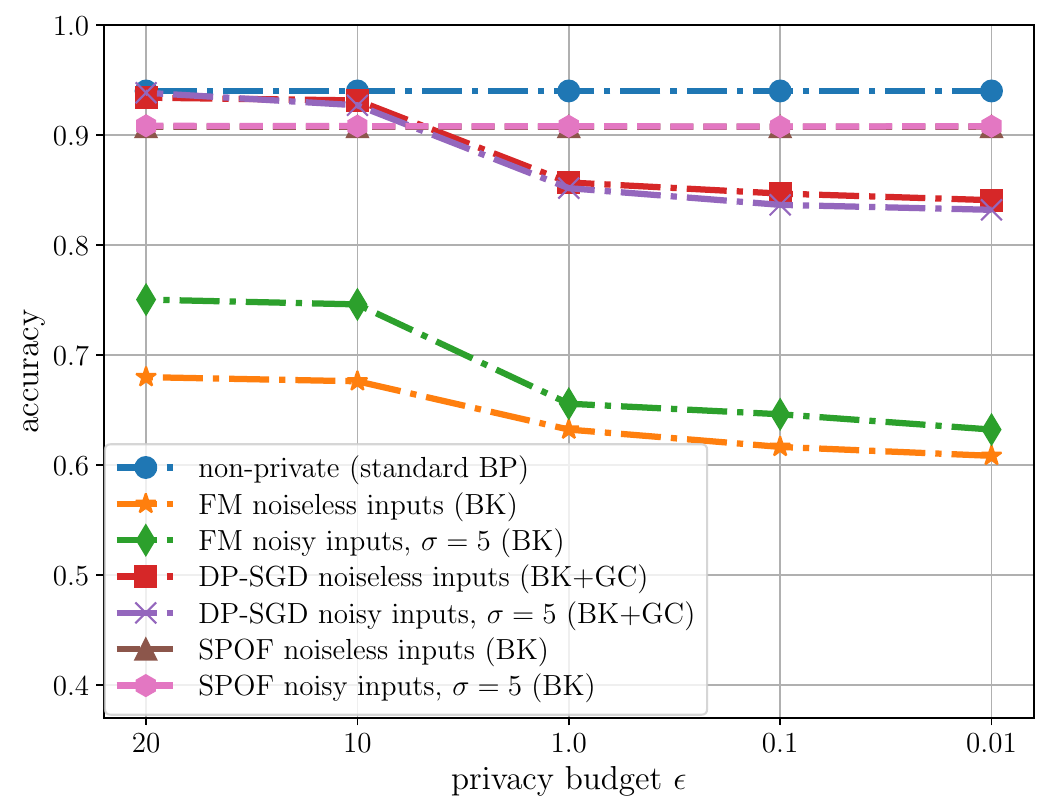}
    \vspace{-0.2cm}
    \caption{Comparison of privacy-utility trade-offs between non-private, DP-SGD, and SPOF based on our DP-WBAN for $m=2$. In the case of FM, we consider an SA trained on the same dataset used for the other schemes.}
    \label{fig:res1}
  \end{minipage}
\end{figure}

We compare SPOF and DP-SGD using the FastDP library that implements BK and GC techniques proposed in \cite{small_cost, group_wise}. As shown in \cite{small_cost, group_wise}, even optimized DP-SGD methods like BK and GC must compute per-sample norms and clipping, which are avoided in SPOF. The computational complexity of {Steps 3 and 5} in Algorithm 1, denoted as $C_{\mathrm{SPOF}}$, is compared with that of {Steps 4 and 6} of Algorithm 2 in Appendix~\ref{app:complexity}, where the latter is denoted as $C_{\mathrm{DP-SGD}}$. Appendix~\ref{app:complexity} gives the arithmetic complexity of gradient norm calculation, clipping, and noise injection, which are fundamental to any DP-SGD algorithm, including the ones in \cite{small_cost, group_wise, strong_dp, zero_redun, JL_proj}. Our analysis shows that {SPOF achieves a $74$\% lower perturbation complexity compared to DP-SGD for the chosen DA, calculated as $(C_{\mathrm{DP-SGD}}-C_{\mathrm{SPOF}})/C_{\mathrm{DP-SGD}}\times 100$}, where DP-SGD's higher complexity is mainly driven by the norm calculation.

To select an appropriate value for the loss stabilization constant in SPOF, we perform a parameter sweep over the range $[0,2.5]$ and evaluate the corresponding reconstruction accuracy. As shown in Fig.~\ref{fig:const_sweep}, the accuracy plateaus near $2.5$, and we therefore assign $c_{j,i} = 2.5$ for all values of $\epsilon$ during training.

From Fig.~\ref{fig:res1}, we notice that when noiseless inputs are used, SPOF achieves $2.90$\% higher accuracy than DP-SGD on average. On the other hand, when sensor noise with s.d. $\sigma=5$ is present, SPOF offers $3.52\%$ mean accuracy gain over DP-SGD, but at a significantly lower training complexity. Thus SPOF is more suitable for DAs used in WBANs which are power constrained. In this example, DP-SGD degrades in performance when level of sensor noise is high, as $\lVert\nabla \hat{L}_j \lVert_2$ is more likely to be larger than $C$ until learning converges. As discussed in Section \ref{sec:loss_dpsgd}, this phenomenon reduces the impact of $b_j$ on DP noise scale. Moreover, our simulations show that the utility of DP-SGD is unaffected by the loss stabilization constant used in SPOF. This is because noise is added in {Step 6} of Algorithm 2 after the gradients are computed in {Step 4}, nullifying the effect of loss stabilization on DP-SGD's accuracy.

Additionally, Fig.~\ref{fig:res1} shows the performance-utility trade-off using FM for the case $m=1$, i.e., we consider an SA, consisting of a single encoder-decoder pair without loss stabilization, which was trained on the same dataset used for training the DA. This randomizing mechanism demonstrates a worse privacy-utility trade-off than SPOF and DP-SGD. However, FM's performances improves in presence of environmental noise, which corroborates Remark \ref{rem:spof_guarantee}.

To comprehend SPOF's utility under noisy inputs, note in Remark~\ref{rem:SPOF_noise} of Appendix~\ref{app:sensor_noise} that, under sufficient environmental noise levels, SPOF requires less DP noise under noisy inputs with probability of at least $0.5$ to achieve the same privacy level as in the noiseless case, preserving utility without degradation. Fig.~\ref{fig:res1} shows that SPOF maintains high reconstruction accuracy even when inputs are corrupted by environmental noise, due to reduced DP noise requirements for at least $50$\% of the noisy data. In the remaining cases, where more DP noise is necessary, SPOF still performs reliably due to the robustness of its stabilized loss function to variations in noise scale. This feature enables SPOF to maintain an almost $\epsilon$-independent accuracy, making it more suitable than DP-SGD for DAs under both noiseless and noisy conditions.

All experiments were done by running $100$ independent Monte Carlo simulations for each value of $\epsilon$. Under this configuration, SPOF achieves an average training time reduction of $57.2$\% relative to DP-SGD across both noisy and noiseless inputs\footnote{{Note that the percentage difference in simulation time does not match the percentage difference in perturbation complexities of SPOF and DP-SGD, since simulation time is affected by additional factors, e.g., gradient computation time, which is not taken in account by the perturbation complexity.}}, calculated as $\left(\text{DP-SGD time} - \text{SPOF time}\right) /{\text{DP-SGD time}} \times 100$. 

\section{Conclusion}
We proposed SPOF to overcome the problem of attaining LDP in multi-user learning systems. Environmental noise, such as EMI, injects noise in training dataset and creates noisy inputs. We analytically {showed} how DP noise scale in SPOF and DP-SGD depends on environmental noise. We also {proposed} novel methods for computing DP-SGD sensitivity and perturbation complexity, facilitating a direct comparison with SPOF. We {demonstrated} that for a fixed privacy budget \(\epsilon\), SPOF achieves better reconstruction accuracy at a smaller training cost compared to DP-SGD. Furthermore, SPOF is robust to environmental noise due to a combination of factors such as reduced DP noise requirements and loss stabilization, whereas DP-SGD's utility degrades under such conditions.

\begin{appendices}
\renewcommand{\thesubsection}{\Roman{subsection}} 
\renewcommand{\thesubsection}{\thesection.\Roman{subsection}}
\renewcommand{\thesubsectiondis}{\thesection.\Roman{subsection}}

\section{Error Introduced by Function Approximation}
\label{app:err_funcApprox}
For training sample $j$ and its $i$-th feature, the error, $E_{j,i} $, caused by replacing the Taylor expansion $L_j$ in (\ref{eq:loss_orig}) with an approximated loss in (\ref{eq:loss_jl_hat}), is given by
\begin{align}
 E_{j,i}\triangleq \sum_{k=1}^2 \sum_{r=3}^\infty \left(\frac{f_{j,i,k}^{(r)}(a)}{r!} (z_{j,i}-a)^r \right).
\end{align}
Let
\[
T_{j,i,k} \triangleq \sum_{r=3}^\infty \frac{f_{j,i,k}^{(r)}(a)}{r!} (z_{j,i} - a)^r,
\]
resulting in an average approximation error for user \(j\) as
\[
E_j \triangleq \frac{1}{n} \sum_{i=1}^n \sum_{k=1}^2 T_{j,i,k} = \sum_{k=1}^2 \left( \frac{1}{n} \sum_{i=1}^n T_{j,i,k} \right).
\]
Also let
  \[
  M_{j,k} \triangleq \max_{z_{j,i}} T_{j,i,k}, \quad m_{j,k} \triangleq \min_{z_{j,i}} T_{j,i,k}.
  \]
Hence, we obtain
\[
m_{j,k} \leq \frac{1}{n} \sum_{i=1}^n T_{j,i,k} \leq M_{j,k}.
\]
Thus, the mean value of \(T_{j,i,k}\) over \(i\) lies within the range \([m_{j,k}, M_{j,k}]\), resulting in an approximation error upper-bound
\[
E_j \leq \sum_{k=1}^2 M_{j,k} = \sum_{k=1}^2 \max_{z_{j,i}} T_{j,i,k}.
\]
Suppose that \( z_{j,i} \in [a - \delta, a + \delta] \) for some \( \delta > 0 \), and there exists a constant \( G > 0 \) such that for all \( r \geq 3 \), \( \left| f_{j,i,k}^{(r)}(a) \right| \leq G \). Then, the magnitude of each term in \( T_{j,i,k} \) is bounded as
\[
\left| \frac{f_{j,i,k}^{(r)}(a)}{r!} (z_{j,i} - a)^r \right| \leq \frac{G \delta^r}{r!}.
\]
Note that the exponential function can be expanded as
\[
e^\delta = \sum_{r=0}^\infty \frac{\delta^r}{r!}
\]
resulting in,
\[
\sum_{r=3}^\infty \frac{\delta^r}{r!} = e^\delta - 1 - \delta - \frac{\delta^2}{2}.
\]
This yields
\[
|T_{j,i,k}| \leq \sum_{r=3}^\infty \frac{G \delta^r}{r!} = G \left( e^\delta - 1 - \delta - \frac{\delta^2}{2} \right),
\]
which gives the bound in (\ref{eq:error}), where the factor of $2$ takes into account the sum over \( k = 1,2 \). \hfill\(\Box\)

\section{SPOF and DP-SGD Sensitivity Derivations}
\label{app:SPOF_DP_SGD_sens}
\subsection{For SPOF}
\label{app:SPOF_sens}
Suppose the two neighboring datasets $\mathcal{D}$ and $\mathcal{D}'$ differ in only the $j$-th data $\mathbf{x}_m$ and $\mathbf{x}_m'$, respectively, the $i$-th element of which is denoted as $x_{j,i}$ and $x_{j,i}'$, respectively. We define the ``sensitivity'' of SPOF as
\begin{align}
	S_{\mathrm{SPOF}}& \triangleq \sum_{i=1}^n \max_{x_{j,i},x_{j,i}'} \sum_{k=1}^2 \sum_{r=0}^2  \left| f_{j,i,k}^{(r)}(a)-  f_{j,i,k}^{(r)'}(a) \right| \label{eq:Sfm1} \\ 
	&\leq \sum_{i=1}^n \max_{x_{j,i},x_{j,i}'} \sum_{k=1}^2 \sum_{r=0}^2  \left(\left| f_{j,i,k}^{(r)}(a) \right| + \left| f_{j,i,k}^{(r)'}(a) \right| \right) \nonumber \\ 
	&\leq 2  \sum_{i=1}^n \max_{x_{j,i}}\sum_{k=1}^2 \sum_{r=0}^2 \left| f_{j,i,k}^{(r)}(a) \right| \label{eq:delFMorig0} \\ 
	& {=2  \sum_{i=1}^n\Delta_{\mathrm{SPOF}}(i)} \label{eq:delFMorig0_2} \\ 
	&\triangleq \Delta_{\mathrm{SPOF}}.
	\label{eq:delFMorig}
\end{align}
We {expand} the $i$-th {term}, $\Delta_{\mathrm{SPOF}}(i)$, {of $ \Delta_{\mathrm{SPOF}}$} in (\ref{eq:delfmi}) of Appendix~\ref{app:sensitivity_optimize} and show using Fig.~\ref{fig:sensitivity1} that {the near-optimal condition for minimizing  $\Delta_{\mathrm{SPOF}}(i)$ occurs at $a=0$}. Since $\sum_{k=1}^2f_{j,i,k}^{(0)}(0)=\log 2$, which is a constant and independent of $x_{j,i}$, it can be removed from the RHS of (\ref{eq:delFMorig0}) when $a=0$, resulting in
\begin{equation}
	\bar{\Delta}_{\mathrm{SPOF}}\triangleq 2\sum_{i=1}^n \max_{x_{j,i}}\sum_{k=1}^2 \sum_{r=1}^2 \left| f_{j,i,k}^{(r)}(0) \right|.
	\label{eq:sensitivity1}
\end{equation}
By rewriting (\ref{eq:sensitivity1}) using the coefficients used in (\ref{eq:lossSimp}), we get
\begin{align}
	\bar{\Delta}_{\mathrm{SPOF}}&= 2\sum_{i=1}^n \max_{x_{j,i}} \left(\left|\alpha_{j,i,2} \right|+\left|\alpha_{j,i,3} \right| \right) 
	\label{eq:delFM2} \\
	&= 2\sum_{i=1}^n \max_{x_{j,i}} \left(\left|0.5-x_{j,i} \right|+ \left|0.5x_{j,i}-0.25 \right| \right) \nonumber \\
	&= \frac{3n}{2}, \label{eq:fm_bound}
\end{align}
and the last equality holds as $0\leq x_{j,i} \leq 1$. \hfill\(\Box\)

\subsection{For DP-SGD}
\label{app: dp_sgd_sens}
The sensitivity of DP-SGD is defined as 
\begin{align}
S_{\mathrm{DP-SGD}} &\triangleq \sum_{i=1}^n \max_{x_{j,i},x_{j,i}'} \left(\left|\nabla \bar{L}_{j,i}- \nabla \bar{L}_{j,i}'\right|\right) \nonumber \\
&\leq \sum_{i=1}^n \max_{x_{j,i}, x_{j,i}'} \left(\left|\nabla \bar{L}_{j,i}'\right|+ \left|\nabla \bar{L}_{j,i}\right|\right) \nonumber \\
&\leq 2\sum_{i=1}^n \max_{x_{j,i}} \left(\left|\nabla \bar{L}_{j,i}\right|\right) \nonumber \\
&= 2 \sum_{i=1}^n \max_{x_{j,i}} \left|\frac{{\frac{e^a}{1+e^a}-x_{j,i}}}{\max\left(1, \frac{\lVert \nabla \hat{L}_j \lVert_2}{C}\right)} \right| \nonumber \\
&\triangleq \Delta_{\mathrm{SGD}}.
\end{align}
{Similarly, in the presence of environmental noise, the corresponding sensitivity is upper-bounded as
\begin{align}
&\sum_{i=1}^n \max_{x_{j,i},x_{j,i}'} \left(\left|\tilde{\nabla} \bar{L}_{j,i}- \tilde{\nabla} \bar{L}_{j,i}'\right|\right) \nonumber \\
&\leq 2 \sum_{i=1}^n \max_{x_{j,i}} \left|\frac{b_j\left(\frac{e^a}{1+e^a}-x_{j,i}\right)}{\max\left(1, \frac{\lVert \nabla \hat{L}_j \lVert_2}{C}\right)} \right| \nonumber \\
&\triangleq \Delta_{\mathrm{SGD}}^{[N]}.
\end{align}
}
In Appendix~\ref{app:sensitivity_optimize} {(see (\ref{eq:delta_sgdi2}) and related discussions)}, we show that the $i$-th {term}, $\Delta_{\mathrm{SGD}}(i)$, of $\Delta_{\mathrm{SGD}}$ can be expressed as
\begin{align*}
\Delta_{\mathrm{SGD}}(i)&\triangleq
\begin{cases}
\Delta_{\mathrm{SGD}(1)}(i), \text{ if } \lVert\nabla \hat{L}_j \lVert_2< C, \\
\Delta_{\mathrm{SGD}(2)}(i), \text{ if } \lVert\nabla \hat{L}_j \lVert_2\geq C,
\end{cases} \\
&=
\begin{cases}
{\frac{2e^a}{1+e^a}}, \text{ if } \lVert\nabla \hat{L}_j \lVert_2< C, \\
2C, \text{ if } \lVert\nabla \hat{L}_j \lVert_2\geq C,
\end{cases}
\end{align*}
and hence
\begin{equation}
\Delta_{\mathrm{SGD}}=
\begin{cases}
{n}, \text{ if } \lVert\nabla \hat{L}_j \lVert_2< C\text{ and }a=0, \\
2nC, \text{ if } \lVert\nabla \hat{L}_j \lVert_2\geq C.
\label{eq:dpsgd_noisl}
\end{cases}
\end{equation}
In Fig.~\ref{fig:sensitivity1} of Appendix~\ref{app:sensitivity_optimize}, we show that $\bar{\Delta}_{\mathrm{SPOF}}(i)$ is {approximately minimized in the vicinity of} $a=0$ for evaluating the loss function coefficients. As a result, we choose $a=0$ in (\ref{eq:dpsgd_noisl}) when $ \lVert\nabla \hat{L}_j \lVert_2< C$ to ensure a fair comparison with SPOF. When $\mathbf{n}_j\neq \mathbf{0}$, we get
\begin{align*}
\Delta_{\mathrm{SGD}}^{[N]}(i)&\triangleq
\begin{cases}
\Delta_{\mathrm{SGD}(1)}^{[N]}(i), \text{ if } \lVert\nabla \hat{L}_j \lVert_2< C, \\
\Delta_{\mathrm{SGD}(2)}^{[N]}(i), \text{ if } \lVert\nabla \hat{L}_j \lVert_2\geq C,
\end{cases} \\
&=
\begin{cases}
{\frac{2b_je^a}{1+e^a}}, \text{ if } \lVert\nabla \hat{L}_j \lVert_2< C, \\
2C, \text{ if } \lVert\nabla \hat{L}_j \lVert_2\geq C,
\end{cases}
\end{align*}
and hence
\begin{equation}
\Delta_{\mathrm{SGD}}^{[N]}=
\begin{cases}
b_j n, \text{ if } \lVert\nabla \hat{L}_j \lVert_2< C\text{ and }a=0, \\
2nC, \text{ if } \lVert\nabla \hat{L}_j \lVert_2\geq C.
\end{cases}
\label{eq:dpsgd_noise}
\end{equation}

\section{SPOF Satisfies $\epsilon$-DP When Applied to a DA}
\label{app:dp_garant}

\subsection{In Presence of Noiseless Inputs}
\label{app:nl_SPOF}

\begin{align}
	\frac{p_{\mathcal{D}}(\mathbf{q})}{p_{\mathcal{D}'}(\mathbf{q})}&= \prod_i \frac{\exp\left(\frac{-\epsilon}{\bar{\Delta}_{\mathrm{SPOF}}}\lvert\sum_{p=2}^3   \alpha_{j,i,p}-q_i\lvert\right)}{\exp\left(\frac{-\epsilon}{\bar{\Delta}_{\mathrm{SPOF}}}\lvert\sum_{p=2}^3   \alpha_{j,i,p}'-q_i\lvert\right)} \nonumber \\
	&= \prod_i \exp\left(\frac{\epsilon}{\bar{\Delta}_{\mathrm{SPOF}}}\left(\bigg|\sum_{p=2}^3  \alpha_{j,i,p}'-q_i\bigg|- \right. \right. \nonumber \\
&~~~ \left. \left. \bigg|\sum_{p=2}^3   \alpha_{j,i,p}-q_i\bigg|\right)\right) \nonumber \\
	&\leq \prod_i \exp\left(\frac{\epsilon \sum_{p=2}^3 \left|{\alpha_{j,i,p}'-\alpha_{j,i,p}} \right|}{\bar{\Delta}_{\mathrm{SPOF}}}\right) \nonumber \\
	&= {\prod_i \exp\left(\frac{\epsilon \sum_{p=2}^3 \left|\alpha_{j,i,p}-\alpha_{j,i,p}' \right|}{\bar{\Delta}_{\mathrm{SPOF}}}\right)} \nonumber \\
	&\leq \prod_i \exp\left(\frac{\epsilon \sum_{p=2}^3 2 \max_{x_{j,i}}\left|\alpha_{j,i,p} \right|}{\bar{\Delta}_{\mathrm{SPOF}}}\right) \nonumber \\	
	&= \exp\left(\frac{\epsilon \sum_{i=1}^n \sum_{p=2}^3 2 \max_{x_{j,i}}\left|\alpha_{j,i,p} \right|}{\bar{\Delta}_{\mathrm{SPOF}}} \right) \label{eq:dp_nonoise} \\
	&\leq e^{\epsilon}, \label{eq:dp_nonoise2} 
\end{align}
where the last inequality holds since, due to (\ref{eq:delFMorig}) and (\ref{eq:delFM2}), the denominator in (\ref{eq:dp_nonoise}) is $\geq$
the numerator. \hfill\(\Box\)

\subsection{In Presence of Noisy Inputs}
\label{app:n_SPOF}
The coefficient of a DA loss function in the presence of noisy inputs is denoted as \( \alpha_{j,i,k}^{[N]} \) in Appendix~\ref{app:noisy_inp_loss}, whereas it is denoted \( \alpha_{j,i,k} \) in the noiseless case. Additionally, \(\bar{\Delta}_{\mathrm{SPOF}}^{[N]}\) represents an SPOF sensitivity upper-bound when environmental noise $\mathbf{n}_j=\mathbf{0}$, with its \(i\)-th component discussed in detail in Appendix~\ref{app:sensitivity_optimize}. If environmental noise $\mathbf{n}_j\neq \mathbf{0}$, DP is still preserved since
\begin{align}
	\frac{p_{\mathcal{D}}(\mathbf{q})}{p_{\mathcal{D}'}(\mathbf{q})}&\leq \exp\left(\frac{\epsilon \sum_{i=1}^n \sum_{p=2}^3 2 \max_{x_{j,i}}\left|\alpha_{j,i,p}^{[N]} \right|}{\bar{\Delta}_{\mathrm{SPOF}}^{[N]}} \right) \label{eq:dp_noisy} \\
	&\leq  e^{\epsilon}. \nonumber	
\end{align}
By substituting $\alpha_{j,i,p}^{[N]}= b_j \alpha_{j,i,p}$ in (\ref{eq:dp_noisy}) we obtain
\begin{align*}
	\frac{p_{\mathcal{D}}(\mathbf{q})}{p_{\mathcal{D}'}(\mathbf{q})}&\leq \exp\left(\frac{\epsilon \sum_{i=1}^n \sum_{p=2}^3 2 \max_{x_{j,i}}\left|\alpha_{j,i,p} \right|}{\bar{\Delta}_{\mathrm{SPOF}}^{[N]}/b_j} \right).
\end{align*} \hfill\(\Box\)

This shows that to match (\ref{eq:dp_nonoise}) and achieve the same level of privacy as in the noiseless case, we should select $d=\bar{\Delta}_{\mathrm{SPOF}}^{[N]}/\epsilon$ when $\mathbf{n}_j\neq \mathbf{0}$, where $\bar{\Delta}_{\mathrm{SPOF}}^{[N]}= b_j \bar{\Delta}_{\mathrm{SPOF}}$. From the discussions related to Fig.~\ref{fig:prob} in Appendix~\ref{app:sensor_noise}, it follows that $\Pr[b_j\leq 1]\geq 0.5$ if the s.d. $\sigma$ of environmental noise is large enough. {This suggests that for at least $50$\% of the training samples, the amount of DP noise can be reduced by choosing $d=\bar{\Delta}_{\mathrm{SPOF}}^{[N]}/\epsilon$ under noisy conditions.}

\section{Impact of Loss Stabilization Constant on SPOF's DP Guarantee}
\label{app:loss_const}
Recall from Section~\ref{sec:DA_loss} that the independent variable of the loss function $\tilde{L}_j$ is given as $z_{j,i}= \hat{\mathbf{W}}_{(:i[j])}^\top\mathbf{h}_j$. We consider modifying the decoder weights {$\hat{\mathbf{W}}_{(:i[j])}^\top$} by adding 
$\mathbf{c}_j$.

The modified loss function variable becomes
\begin{align*}
\hat{z}_{j,i}&\triangleq \left(\hat{\mathbf{W}}_{(:i[j])}^\top+ \mathbf{c}_j^\top \right) \mathbf{h}_j \\
&= \hat{\mathbf{W}}_{(:i[j])}^\top \mathbf{h}_j+ \mathbf{c}_j^\top \mathbf{h}_j \\
&= z_{j,i}+ c_j. 
\end{align*}
This adjustment allows for tuning the privacy-utility trade-off without compromising DP guarantees. The loss function in (\ref{eq:lossSimp}) can be rewritten as
\begin{align}
&n \alpha_{j,i,1} + \sum_{i=1}^n  \left(\alpha_{j,i,2}(z_{j,i}+c_j) + \alpha_{j,i,3}(z_{j,i}+c_j)^2 \right) \nonumber \\ 
&= n \alpha_{j,i,1}+ c_j^2 \sum_{i=1}^n \alpha_{j,i,3} + \sum_{i=1}^n \left(\alpha_{j,i,2}(1+2c_j) z_{j,i} + \alpha_{j,i,3} z_{j,i}^2 \right).
  \label{eq:lossSimp2} 
\end{align}
We define a new sensitivity to reflect the loss stabilization constant according to
\begin{align}
	\hat{\Delta}_{\mathrm{SPOF}}&\triangleq 2\max_{x_{j,i}} \left(c_j^2 \sum_{i=1}^n \alpha_{j,i,3} \right. \nonumber \\ 
	&~~~ \left. + \sum_{i=1}^n \left(|\alpha_{j,i,2} (1+2c_j)|+ |\alpha_{j,i,3}| \right)\right) \nonumber \\
	&\leq c_j^2 \frac{n}{2}+ 2\sum_{i=1}^n \max_{x_{j,i}} \left(|\alpha_{j,i,2} (1+2c_j)|+ |\alpha_{j,i,3}| \right) \nonumber \\
	&= c_j^2 \frac{n}{2}+ 2\sum_{i=1}^n \max_{x_{j,i}} \left(|(0.5-x_{j,i}) (1+2c_j)|\right. \nonumber \\
&~~~ \left. + |0.5x_{j,i}-0.25| \right)  \nonumber \\
	&\leq c_j^2 \frac{n}{2}+ 2n (0.5(1+2c_j)+ 0.25 ) \nonumber \\
	&= \frac{3n}{2}+ n c_j\left(\frac{c_j}{2}+2 \right)  \nonumber \\
	&= \bar{\Delta}_{\mathrm{SPOF}}+ n c_j\left(\frac{c_j}{2}+2 \right), \label{eq:hatfm} 
\end{align}
where $\bar{\Delta}_{\mathrm{SPOF}}$ is SPOF sensitivity without {loss stabilization} constant as shown in (\ref{eq:delFM2}). Note that $\alpha_{j,i,1}$ is a constant and hence not taken into account for deriving the bound in (\ref{eq:hatfm}), and this exclusion also applies to (\ref{eq:sensitivity1}). Moreover, when $c_j=0$, \emph{i.e.,} $\mathbf{c}_j=\mathbf{0}$, $\hat{\Delta}_{\mathrm{SPOF}}=\bar{\Delta}_{\mathrm{SPOF}}$.

Let $p_{\mathcal{D}}(\mathbf{q})$ be the PDF of the randomizing mechanism $\mathcal{M}(\mathcal{D},f(\mathcal{D}),\epsilon)$, where the query functions {are} $f(\mathcal{D})_i=\sum_{p=2}^3 \hat{\alpha}_{j,i,p}+ c_j^2 \alpha_{j,i,3}$ {and} $f(\mathcal{D}')_i=\sum_{p=2}^3 \hat{\alpha}_{j,i,p}'+ c_j^2 \alpha_{j,i,3}'$ {with the coefficients $\hat{\alpha}_{j,i,1}\triangleq \alpha_{j,i,1}+ c_j^2 \alpha_{j,i,3}$, $\hat{\alpha}_{j,i,2}\triangleq \alpha_{j,i,2}(1+2c_j)$, and $\hat{\alpha}_{j,i,3}\triangleq \alpha_{j,i,3}$}. 
By applying the proof technique of Appendix~\ref{app:nl_SPOF}, the effect of the loss stabilization constant on the DP guarantee is derived in (\ref{eq:edp_const}), implying that adding a loss stabilization constant still ensures SPOF's $\epsilon$-DP guarantee. \hfill\(\Box\)

\begin{figure*}[t]
\begin{align}
	\frac{p_{\mathcal{D}}(\mathbf{q})}{p_{\mathcal{D}'}(\mathbf{q})}&\leq \exp\left(\frac{\epsilon \sum_{i=1}^n 2 \max_{x_{j,i}} \left(\sum_{p=2}^3 \left|\hat{\alpha}_{j,i,p} \right|+ c_j^2 \left|\alpha_{j,i,3} \right| \right)}{\hat{\Delta}_{\mathrm{SPOF}}} \right) \nonumber \\
	&= \exp\left(\frac{\epsilon \sum_{i=1}^n 2 \max_{x_{j,i}} \left(|\alpha_{j,i,2}(1+2c_j)|+ |\alpha_{j,i,3}| + c_j^2 \left|\alpha_{j,i,3} \right| \right)}{\hat{\Delta}_{\mathrm{SPOF}}} \right) \nonumber \\
	&= \exp\left(\frac{\epsilon \sum_{i=1}^n \sum_{p=2}^3 2 \max_{x_{j,i}} |\alpha_{j,i,p}| + \epsilon \sum_{i=1}^n 2 \max_{x_{j,i}} \left(|\alpha_{j,i,2} c_j|+ c_j^2 \left|\alpha_{j,i,3} \right| \right)}{\hat{\Delta}_{\mathrm{SPOF}}} \right) \nonumber \\
	&\leq  e^{\epsilon}. \label{eq:edp_const}
\end{align}
\end{figure*}

\section{Adaptation of DA Loss Function to Noisy Inputs}
\label{app:noisy_inp_loss}
Let $\mathbf{W}_{j(i:)}^\top \in \mathbb{R}^n$ be the $i$-th row of learnable parameter $\mathbf{W}_j^\top$ of the $j$-th encoder. In {the presence of environmental noise, let $\hat{\mathbf{h}}$ be the output of the $j$-th encoder} with the $i$-th element, $\hat{h}_{j,i}$, expressed as
\begin{align*}
	\hat{h}_{j,i}&\triangleq \sigma(\mathbf{W}_{j(i:)}^\top (\mathbf{x}_j+\mathbf{n}_j)) \nonumber \\
	&=\sigma(\mathbf{W}_{j(i:)}^\top \mathbf{x}_j+\mathbf{W}_{j(i:)}^\top \mathbf{n}_j) \nonumber \\
	&= \frac{1}{1+e^{-\mathbf{W}_{j(i:)}^\top \mathbf{x}_j}e^{-\mathbf{W}_{j(i:)}^\top \mathbf{n}_j}} \nonumber \\
	&= e^{\mathbf{W}_{j(i:)}^\top \mathbf{n}_j} \frac{1}{e^{\mathbf{W}_{j(i:)}^\top \mathbf{n}_j}+e^{-\mathbf{W}_{j(i:)}^\top \mathbf{x}_j}} \nonumber \\
	&= e^{\mathbf{W}_{j(i:)}^\top \mathbf{n}_j} \frac{1}{1+e^{-\mathbf{W}_{j(i:)}^\top \mathbf{x}_j}+e^{\mathbf{W}_{j(i:)}^\top \mathbf{n}_j}-1} \nonumber \\
	&= e^{\mathbf{W}_{j(i:)}^\top \mathbf{n}_j} \frac{\left(1+e^{-\mathbf{W}_{j(i:)}^\top \mathbf{x}_j}\right)^{-1}}{1+\left(e^{\mathbf{W}_{j(i:)}^\top \mathbf{n}_j}-1\right) \left(1+e^{-\mathbf{W}_{j(i:)}^\top \mathbf{x}_j}\right)^{-1}} \nonumber \\
	&=  \frac{e^{\mathbf{W}_{j(i:)}^\top \mathbf{n}_j}}{1+\left(e^{\mathbf{W}_{j(i:)}^\top \mathbf{n}_j}-1\right) h_{j,i}} h_{j,i},
\end{align*}
Now, let
\begin{equation}
b_{j,i}\triangleq \frac{e^{\mathbf{W}_{j(i:)}^\top \mathbf{n}_j}}{1+\left(e^{\mathbf{W}_{j(i:)}^\top \mathbf{n}_j}-1\right) h_{j,i}} \geq 0,
\label{eq:bji}
\end{equation}
as $h_{j,i}\in (0,1)$. Let
\begin{align}
\hat{z}_{j,i} &\triangleq \hat{\mathbf{W}}_{(:i[j])}^\top\hat{\mathbf{h}}_j \nonumber \\
&=\sum_{r=1}^l \hat{W}_{(j-1)l+r,i} \hat{h}_{j,r} \nonumber \\ 
&=\sum_{r=1}^l \hat{W}_{(j-1)l+r,i} h_{j,r} \frac{e^{\mathbf{W}_{j(:r)}^\top \mathbf{n}_j}}{1+\left(e^{\mathbf{W}_{j(:r)}^\top \mathbf{n}_j}-1\right) h_{j,r}} \nonumber \\
&=\sum_{r=1}^l \hat{W}_{(j-1)l+r,i} h_{j,r} b_{j,r} \nonumber \\
&=\left(b_{j,1}+ \cdots + b_{j,l} \right) \sum_{r=1}^l \hat{W}_{(j-1)l+r,i} h_{j,r} \nonumber \\
&~~~ - \sum_{r=1}^l \left(b_{j,1}+\cdots+ b_{j,l} \right)_{-r} \hat{W}_{(j-1)l+r,i} h_{j,r} \nonumber \\
&=\left(b_{j,1}+ \cdots + b_{j,l} \right)\hat{\mathbf{W}}_{(:i[j])}^\top\mathbf{h}_j \nonumber \\
&~~~ - \sum_{r=1}^l \left(b_{j,1}+\cdots+ b_{j,l} \right)_{-r} \hat{W}_{(j-1)l+r,i} h_{j,r},
\label{eq:zji}
\end{align}
where $\left(b_{j,1}+\cdots+ b_{j,l} \right)_{-r}$ is obtained by removing the $r$-th term from $\left(b_{j,1}+\cdots+ b_{j,l} \right)$. Let 
\begin{equation}
	b_j\triangleq b_{j,1}+ \cdots + b_{j,l} 
	\label{eq:bj}
\end{equation}
and
\begin{align}
	t_{j,i}&\triangleq \sum_{r=1}^l \left(b_{j,1}+\cdots+ b_{j,l} \right)_{-r} \hat{W}_{(j-1)l+r,i} h_{j,r}.
\end{align}
Thus, (\ref{eq:zji}) can be simplified as
\begin{align}
\hat{z}_{j,i}&= b_j\hat{\mathbf{W}}_{(:i[j])}^\top \mathbf{h}_j-t_{j,i} \nonumber \\
&= b_jz_{j,i}-t_{j,i}.
\label{eq:zji2}
\end{align}
If environmental noise $\mathbf{n}_j\neq \mathbf{0}$, a modified version of (\ref{eq:loss_jl_hat}) is expressed using the loss function $\hat{L}_j^{[N]}$ in (\ref{eq:loss_jl_hat2}) of Appendix~\ref{app:noisy_inp_loss}, with polynomial coefficients $\alpha_{j,i,1}^{[N]}= \alpha_{j,i,1}$, $\alpha_{j,i,2}^{[N]}= b_j \alpha_{j,i,2}$, and $\alpha_{j,i,3}^{[N]}=b_j \alpha_{j,i,3}$. However, if $\mathbf{n}_j=\mathbf{0}$, then $b_{j,i}=1{~\forall j,i}$ and 
\begin{align}
\hat{z}_{j,i}&={l \hat{\mathbf{W}}_{(:i[j])}^\top\mathbf{h}_j - (l-1)\sum_{r=1}^l \hat{W}_{(j-1)l+r,i} h_{j,r}} \nonumber \\
&=l \hat{\mathbf{W}}_{(:i[j])}^\top\mathbf{h}_j - (l-1)\hat{\mathbf{W}}_{(:i[j])}^\top\mathbf{h}_j \nonumber \\
&=\hat{\mathbf{W}}_{(:i[j])}^\top\mathbf{h}_j \nonumber \\
&=z_{j,i},
\end{align}
implying that the DA loss function reverts to its original form in (\ref{eq:loss_jl_hat}).

Let $L^{[N]}$ be a DA loss function when noisy inputs are taken into consideration, where
\begin{align}
L^{[N]}&\triangleq \sum_{j=1}^m L_j^{[N]} \nonumber \\ 
&=\sum_{j=1}^m \sum_{i=1}^n \left\{ \left(x_{j,i} \log \left(1+e^{-\hat{\mathbf{W}}_{(:i[j])}^\top\hat{\mathbf{h}}_j}\right) \right. \right. \nonumber \\
&~~~ \left. \left. +(1-x_{j,i})\log \left(1+e^{\hat{\mathbf{W}}_{(:i[j])}^\top\hat{\mathbf{h}}_j}\right) \right) \right\} \nonumber \\
&= \sum_{j=1}^m \sum_{i=1}^n \sum_{k=1}^2 f_{j,i,k}\left({\hat{z}}_{j,i}\right),
\end{align}
and $L_j^{[N]}$ is the loss function for the $j$-th user which can also be expressed using Taylor series as
\begin{align}
L_j^{[N]}&\triangleq \sum_{i=1}^n \sum_{k=1}^2 \sum_{r=0}^\infty \frac{f_{j,i,k}^{(r)}(a)}{r!} \left(\hat{z}_{j,i}-a \right)^r. 
\label{eq:jln}
\end{align}
Additionally, if $\mathbf{n}_j\neq \mathbf{0}$, by substituting the RHS of (\ref{eq:zji2}) in (\ref{eq:jln}) we get
\begin{align}
L_j^{[N]}&= \sum_{i=1}^n \sum_{k=1}^2 \sum_{r=0}^\infty b_j^r\frac{f_{j,i,k}^{(r)}(a)}{r!} \left(z_{j,i}-\frac{t_{j,i}+a}{b_j} \right)^r \nonumber \\
&= \sum_{i=1}^n \sum_{k=1}^2 \sum_{r=0}^\infty \frac{\tilde{f}_{j,i,k}^{(r)}(a)}{r!} \left(z_{j,i}-\frac{t_{j,i}+a}{b_j} \right)^r,
\end{align}
where $\tilde{f}_{j,i,k}^{(r)}(a)=b_j^r f_{j,i,k}^{(r)}(a)$. An approximated version of $L_j^{[N]}$ is given as
\begin{align}
\hat{L}_j^{[N]}&= \sum_{i=1}^n \sum_{k=1}^2 \left(\tilde{f}_{j,i,k}^{(0)}(a) + \tilde{f}_{j,i,k}^{(1)}(a)\left(z_{j,i}-\frac{t_{j,i}+a}{b_j} \right) \right. \nonumber \\
&~~~ \left. + \tilde{f}_{j,i,k}^{(2)}(a) \left(z_{j,i}-\frac{t_{j,i}+a}{b_j} \right)^2 \right),
\label{eq:loss_jl_hat2}
\end{align}
with Taylor series coefficients $\tilde{f}_{j,i,1}^{(0)}(a)= x_{j,i} \log (1+e^{-a})$, $\tilde{f}_{j,i,1}^{(1)}(a)=-\frac{b_j x_{j,i}}{1+e^{a}}$, $\tilde{f}_{j,i,1}^{(2)}(a)=b_j x_{j,i} \frac{e^{a}}{(1+e^{a})^2}$, $\tilde{f}_{j,i,2}^{(0)}(a)=(1-x_{j,i}) \log (1+e^{a})$, $\tilde{f}_{j,i,2}^{(1)}(a)=\frac{ b_j (1-x_{j,i})}{1+e^{-a}}$, $\tilde{f}_{j,i,2}^{(2)}(a)=-b_j (1-x_{j,i}) \frac{e^{-a}}{(1+e^{-a})^2}$. Based on these coefficients, (\ref{eq:loss_jl_hat2}) is further simplified as
\begin{align}
\hat{L}_j^{[N]}&= \sum_{i=1}^n  \left( \log \left(\left(\frac{1+e^{-a}}{1+e^a}\right)^{x_{j,i}}(1+e^a)\right) \right. \nonumber \\
&~~~ \left.  + \left(\frac{b_j}{1+e^{-a}}-b_j x_{j,i}\right)  \left(z_{j,i}-\frac{t_{j,i}+a}{b_j} \right) \right. \nonumber \\
 &~~~~~ \left. + \left(\frac{b_j(2x_{j,i}-1)}{2+e^a+e^{-a}} \right)  \left(z_{j,i}-\frac{t_{j,i}+a}{b_j} \right)^2 \right).
 \label{eq:L_hat_N}
\end{align}
By choosing $a=0$ in (\ref{eq:L_hat_N}), the result follows. \hfill\(\Box\)

\section{Effect of Noisy Inputs on User Privacy}
\label{app:sensor_noise}
We first derive the PDF of $b_{j,i}$ shown in (\ref{eq:bji}), and use it to obtain the PDFs of the maximum and minimum values of $b_j$ defined in (\ref{eq:bj}). These PDFs facilitate a statistical analysis of the likelihood that environmental noise $\mathbf{n}_j$ reduces the sensitivity of SPOF, thereby requiring less DP noise to be added without compromising privacy guarantee. {Although the derivations in Appendix~\ref{app:senn_fm} are related to SPOF, the content of Appendix~\ref{app:comb} applies also to DP-SGD as the noise parameter $b_j$ is used in (\ref{eq:dpsgd_noise}).}

\subsection{Relating Environmental Noise to SPOF}
\label{app:senn_fm}
When $\mathbf{n}_j\neq \mathbf{0}$, the approximated DA loss coefficients are given by $\tilde{f}_{j,i,k}^{(r)}(0)$, {whereas it is denoted as} $f_{j,i,k}^{(r)}(0)$ {if $\mathbf{n}_j= \mathbf{0}$}. Let $b_j\leq b_\max$, and recall that $0\leq x_{j,i}\leq 1$. Hence, we get
\begin{align}
	&\bar{\Delta}_{\mathrm{SPOF}}^{[N]}(i)-\bar{\Delta}_{\mathrm{SPOF}}(i) \nonumber \\
   	& \leq 2 \left(\max_{x_{j,i}} \sum_{k=1}^2 \sum_{r=1}^2 \left|\tilde{f}_{j,i,k}^{(r)}(0) \right|- \min_{x_{j,i}} \sum_{k=1}^2 \sum_{r=1}^2 \left|f_{j,i,k}^{(r)}(0) \right| \right) \nonumber \\
	&= 2 \max_{x_{j,i}} \left(\frac{|-b_j x_{j,i}|}{2} +\frac{|b_j x_{j,i}|}{4} +\frac{|b_j(1-x_{j,i})|}{2} .\right. \nonumber \\
   	&~~~ \left. + \frac{|b_j(x_{j,i}-1)|}{4} \right) \nonumber \\
  &~~~ - 2 \min_{x_{j,i}} \left(\frac{|-x_{j,i}|}{2} +\frac{|x_{j,i}|}{4} + \frac{|1-x_{j,i}|}{2}+ \frac{|x_{j,i}-1|}{4} \right) \nonumber \\
	&=  \max_{x_{j,i}} \left(\frac{3b_j x_{j,i}}{2} +\frac{3b_j(1-x_{j,i})}{2} \right)-  \nonumber \\
	&~~~ \min_{x_{j,i}} \left(\frac{3x_{j,i}}{2} + \frac{3(1-x_{j,i})}{2} \right) \nonumber \\
	&= \max_{b_j} \frac{3 b_j}{2} -\frac{3}{2} \nonumber \\
	&= \max_{b_j}\frac{3(b_j-1)}{2} \label{eq:fmdiff1} \\
	&= \frac{3(b_\max-1)}{2}, 
\end{align}
implying that $b_\max< 1$ is a sufficient condition to ensure $\bar{\Delta}_{\mathrm{SPOF}}^{[N]}< \bar{\Delta}_{\mathrm{SPOF}}$. Thus,
\begin{equation}
\Pr\left[\bar{\Delta}_{\mathrm{SPOF}}^{[N]}(i)< \bar{\Delta}_{\mathrm{SPOF}}(i)\right] \geq \Pr[b_\max< 1].
\label{eq:ineq}
\end{equation}
Since $b_\max$ is a continuous r.v., $\Pr[b_\max\leq 1]=\Pr[b_\max< 1]$ as $\Pr[b_\max= 1]=0$. Therefore, (\ref{eq:ineq}) can also be formulated as
\begin{equation}
\Pr\left[\bar{\Delta}_{\mathrm{SPOF}}^{[N]}(i)< \bar{\Delta}_{\mathrm{SPOF}}(i)\right]\geq \Pr[b_\max\leq 1].
\label{eq:ineq2}
\end{equation}

\noindent \textbf{Please refer to the supplemental materials for the remaining appendices.}

\bibliographystyle{IEEEtran}
\bibliography{DP_Bib}

\begin{thebibliography}{10}
\providecommand{\url}[1]{#1}
\csname url@samestyle\endcsname
\providecommand{\newblock}{\relax}
\providecommand{\bibinfo}[2]{#2}
\providecommand{\BIBentrySTDinterwordspacing}{\spaceskip=0pt\relax}
\providecommand{\BIBentryALTinterwordstretchfactor}{4}
\providecommand{\BIBentryALTinterwordspacing}{\spaceskip=\fontdimen2\font plus
\BIBentryALTinterwordstretchfactor\fontdimen3\font minus
  \fontdimen4\font\relax}
\providecommand{\BIBforeignlanguage}[2]{{%
\expandafter\ifx\csname l@#1\endcsname\relax
\typeout{** WARNING: IEEEtran.bst: No hyphenation pattern has been}%
\typeout{** loaded for the language `#1'. Using the pattern for}%
\typeout{** the default language instead.}%
\else
\language=\csname l@#1\endcsname
\fi
#2}}
\providecommand{\BIBdecl}{\relax}
\BIBdecl

\bibitem{wban1}
E.~El-Adawi, E.~Essa, and E.~Handosa, ``Wireless body area sensor networks
  based human activity recognition using deep learning,'' in \emph{Sci Rep.
  2024 Feb 1;14(1):2702. doi: 10.1038/s41598-024-53069-1.}, 2024.

\bibitem{wban2}
C.~Chiu, T.~Nguyen, S.~Wang, B.~Kim, and K.~Kim, ``Active learning for
  wban-based health monitoring,'' in \emph{Proc. of the 25th Intl. Symp. on
  Theory, Algorithmic Foundations, and Protocol Design for Mobile Networks and
  Mobile Computing}, 2024.

\bibitem{jsan11040067}
M.~Yaghoubi, K.~Ahmed, and Y.~Miao, ``Wireless body area network ({WBAN}): A
  survey on architecture, technologies, energy consumption, and security
  challenges,'' \emph{Journal of Sensor and Actuator Networks}, vol.~11, no.~4,
  2022.

\bibitem{zanella2014internet}
A.~Zanella, N.~Bui, A.~Castellani, L.~Vangelista, and M.~Zorzi, ``Internet of
  things for smart cities,'' \emph{IEEE Internet of Things journal}, vol.~1,
  no.~1, pp. 22--32, 2014.

\bibitem{calandrino2011privacy}
J.~A. Calandrino, A.~Kilzer, A.~Narayanan, E.~W. Felten, and V.~Shmatikov,
  ``{Y}ou might also like: Privacy risks of collaborative filtering,'' in
  \emph{Proceedings of the 2011 IEEE Symposium on Security and Privacy}, 2011,
  pp. 231--246.

\bibitem{feng2022inferential}
C.~Feng and P.~Venkitasubramaniam, ``Inferential separation for privacy:
  Irrelevant statistics and quantization,'' \emph{IEEE Trans. Inf. Forensics
  Security}, vol.~17, no.~9, pp. 2241--2255, 2022.

\bibitem{optimalDP}
Q.~Geng and P.~Viswanath, ``The optimal noise-adding mechanism in differential
  privacy,'' \emph{IEEE Trans. Inf. Theory}, vol.~62, no.~2, pp. 925--951,
  2016.

\bibitem{wei2020federated}
K.~Wei, J.~Li, M.~Ding, C.~Ma, H.~Yang, F.~Farokhi, S.~Jin, T.~Q.~S. Quek, and
  H.~V. Poor, ``Federated learning with differential privacy: Algorithms and
  performance analysis,'' \emph{IEEE Trans. Inf. Forensics Security}, vol.~15,
  pp. 3454--3469, 2020.

\bibitem{hu2022one}
D.~Ye, S.~Shen, T.~Zhu, B.~Liu, and W.~Zhou, ``One parameter
  defense—defending against data inference attacks via differential
  privacy,'' \emph{IEEE Trans. Inf. Forensics Security}, vol.~17, pp.
  1466--1480, 2022.

\bibitem{fm1}
J.~Zhang, Z.~Zhang, X.~Xiao, Y.~Yang, and M.~Winslett, ``Functional mechanism:
  Regression analysis under differential privacy,'' in \emph{The 38th Intl.
  Conf. on Very Large Data Bases}, 2012.

\bibitem{fm2}
N.~Phan, Y.~Wang, X.~Wu, and D.~Dou, ``Differential privacy preservation for
  deep auto-encoders: An application of human behavior prediction,'' in
  \emph{Proc. of the 30th AAAI Conf. on Artificial Intelligence}, 2016.

\bibitem{fm3}
J.~Ding, X.~Zhang, X.~Li, J.~Wang, R.~Yu, and M.~Pan, ``Differentially private
  and fair classification via calibrated functional mechanism,'' in \emph{Proc.
  of the 34th AAAI Conf. on Artificial Intelligence}, 2020.

\bibitem{drasic}
E.~Diao, J.~Ding, and V.~Tarokh, ``{DRASIC:} distributed recurrent autoencoder
  for scalable image compression,'' in \emph{2020 Data Compression Conference
  (DCC)}, 2020, pp. 3--12.

\bibitem{li2023taskaware}
P.~Li, S.~K. Ankireddy, R.~Zhao, H.~N. Mahjoub, E.~M. Pari, U.~Topcu, S.~P.
  Chinchali, and H.~Kim, ``Task-aware distributed source coding under dynamic
  bandwidth,'' in \emph{Adv. in Neural Inf. Processing Systems (NeurIPS)},
  2023.

\bibitem{fitbit}
dataset, ``Fitbit fitness tracker data,'' [Online]. Available:
  https://www.kaggle.com/datasets/arashnic/fitbit.

\bibitem{noise}
G.~Ellis, \emph{Noise in the Luenberger Observer, Observers in Control
  Systems}.\hskip 1em plus 0.5em minus 0.4em\relax Academic Press, 2002.

\bibitem{noise2}
D.~Li, Y.~Wang, J.~Wang, C.~Wang, and Y.~Duan, \emph{Recent advances in sensor
  fault diagnosis: A review}.\hskip 1em plus 0.5em minus 0.4em\relax Sensors
  and Actuators A: Physical, Elsevier, 2020.

\bibitem{dpsgd1}
M.~Abadi, H.~McMahan, A.~Chu, I.~Mironov, L.~Zhang, I.~Goodfellow, and
  K.~Talwar, ``Deep learning with differential privacy,'' \emph{23rd ACM Conf.
  on Computer and Communications Security}, 2016.

\bibitem{dpsgd2}
J.~Du, S.~Li, X.~Chen, S.~Chen, and M.~Hong, ``Dynamic differential-privacy
  preserving {SGD},'' in \emph{39th Intl. Conf. on Machine Learning (ICML)},
  2022.

\bibitem{dpsgd3}
Y.~Zhou, Z.~Wu, and A.~Banerjee, ``Bypassing the ambient dimension: Private
  {SGD} with gradient subspace identification,'' in \emph{Intl. Conf. on
  Learning Representations (ICLR)}, 2021.

\bibitem{dpsgd4}
Y.~Ma, T.~V. Marinov, and T.~Zhang, ``Dimension independent generalization of
  {DP-SGD} for overparameterized smooth convex optimization,'' [Online].
  Available: arXiv.org, arXiv:2206.01836 [cs.LG], 2022.

\bibitem{dpsgd5}
A.~Thudi, H.~Jia, C.~Meehan, I.~Shumailov, and N.~Papernot, ``Gradients look
  alike: Sensitivity is often overestimated in {DP-SGD},'' in \emph{33rd
  {USENIX} Security Symposium}, 2024.

\bibitem{small_cost}
Z.~Bu, Y.~Wang, S.~Zha, and G.~Karypis, ``Differentially private optimization
  on large model at small cost,'' in \emph{32nd Intl. Conf. on Machine Learning
  (ICML)}, 2023.

\bibitem{group_wise}
J.~He, X.~Li, D.~Yu, H.~Zhang, J.~Kulkarni, Y.~T. Lee, A.~Backurs, N.~Yu, and
  J.~Bian, ``Exploring the limits of differentially private deep learning with
  group-wise clipping,'' in \emph{Intl. Conf. on Learning Representations
  (ICLR)}, 2023.

\bibitem{DP_ZO}
X.~Tang, A.~Panda, M.~Nasr, S.~Mahloujifar, and P.~Mittal, ``Private
  fine-tuning of large language models with zeroth-order optimization,''
  \emph{Transactions on Machine Learning Research}, 2025.

\bibitem{dwork1}
C.~Dwork and A.~Roth, ``The algorithmic foundations of differential privacy,''
  \emph{Foundations and Trends in Theoretical Computer Science}, vol.~9, no.
  3-4, p. 211–407, 2014.

\bibitem{dwork2}
C.~Dwork, F.~McSherry, K.~Nissim, and A.~Smith, ``Calibrating noise to
  sensitivity in private data analysis,'' in \emph{Theory of
  Cryptography}.\hskip 1em plus 0.5em minus 0.4em\relax Springer Berlin
  Heidelberg, 2006.

\bibitem{Bengio06}
Y.~Bengio, P.~Lamblin, D.~Popovici, and H.~Larochelle, ``Greedy layer-wise
  training of deep networks,'' \emph{Adv. in Neural Inf. Processing Systems
  (NeurIPS)}, 2006.

\bibitem{parallelDP}
J.~Smith, H.~J. Asghar, G.~Gioiosa, S.~Mrabet, S.~Gaspers, and P.~Tyler,
  ``Making the most of parallel composition in differential privacy,''
  \emph{arXiv preprint arXiv:2109.09078}, 2021.

\bibitem{strong_dp}
X.~Li, F.~Tramer, P.~Liang, and T.~Hashimoto, ``Large language models can be
  strong differentially private learners,'' in \emph{Intl. Conf. on Learning
  Representations (ICLR)}, 2022.

\bibitem{zero_redun}
Z.~Bu, J.~Chiu, R.~Liu, Y.~Wang, S.~Zha, and G.~Karypis, ``Zero redundancy
  distributed learning with differential privacy,'' in \emph{ICLR 2023 Workshop
  on Pitfalls of limited data and computation for Trustworthy ML}, 2023.

\bibitem{JL_proj}
Z.~Bu, S.~Gopi, Y.~T. Lee, H.~Shen, J.~Kulkarni, and U.~Tantipongpipat, ``Fast
  and memory efficient differentially private-{SGD} via {JL} projections,''
  \emph{Adv. in Neural Inf. Processing Systems (NeurIPS)}, 2021.

\bibitem{Fog}
A.~Fog, ``Lists of instruction latencies, throughputs and micro-operation
  breakdowns for {I}ntel, {AMD}, and {VIA} {CPU}s,'' [Online]. Available:
  https://www.agner.org/optimize/instruction\_tables.pdf, 2022.

\bibitem{noise_comp}
K.~Ganesan, V.~Karyofyllis, J.~Attai, A.~Hamoda, and N.~Jerger, ``{DINAR}:
  Enabling distribution agnostic noise injection in machine learning
  hardware,'' \emph{Hardware and Architectural Support for Security and
  Privacy}, 2023.

\bibitem{squareroot}
R.~L. Burden and J.~D. Faires, \emph{Numerical Analysis}.\hskip 1em plus 0.5em
  minus 0.4em\relax Cengage Learning, 2010.

\end{thebibliography}

\newpage

\noindent \textbf{Supplemental material (continued from Appendix \ref{app:senn_fm}):}

\subsection{Combination of DA Learnable Parameter and Environmental Noise}
\label{app:comb}
Note that the elements of $\mathbf{n}_j$ are i.i.d. Gaussian r.v.s sampled from $\mathcal{N}(0,\sigma^2)$. For a given DA learnable parameter $\mathbf{W}_{j(i:)}$, by applying the property of linear combinations of Gaussian variables, \(\mathbf{W}_{j(i:)}^\top\mathbf{n}_j \) is also a Gaussian r.v. with 
\[
\mathbb{E}[\mathbf{W}_{j(i:)}^\top\mathbf{n}_j] = 0
\]
and
\vspace{-0.5cm}
\begin{align*}
	\mathrm{Var}\left(\mathbf{W}_{j(i:)}^\top\mathbf{n}_j \right)&= \mathrm{Var} \left(\sum_k W_{j(:i),k} n_{j,k} \right) \nonumber \\
	&= \sigma^2 \lVert\mathbf{W}_{j(i:)}\lVert_2,
\end{align*}
and hence \( \mathbf{W}_{j(i:)}^\top\mathbf{n}_j \sim \mathcal{N}(0, \sigma^2 \lVert\mathbf{W}_{j(i:)}\lVert_2) \). We denote $D = \mathbf{W}_{j(i:)}^\top\mathbf{n}_j$, a r.v. $X=b_{j,i}$ with $\mathbf{n}_j$ being the source of randomness in (\ref{eq:bji}), and $h=h_{j,i}$. Then, $b_{j,i}$ can also be expressed as
\begin{equation}
X = \frac{\exp(D)}{1 + h(\exp(D) - 1)},
\label{eq:X_and_a}
\end{equation}
where $h$ is treated as constant bounded between $[0, 1]$. Our goal is to compute the PDF of $X$ using the PDF of $D$ given by
\[
f_D(d) = \frac{1}{\sqrt{2\pi \tilde{\sigma}^2}} \exp\left(-\frac{d^2}{2\tilde{\sigma}^2}\right),
\]
where $\tilde{\sigma}^2=\sigma^2 \lVert\mathbf{W}_{j(i:)}\lVert_2$. 

{
Note that \( D \sim \mathcal{N}(\mu, \tilde{\sigma}^2) \), $\mu=0$, and let \( Y = \exp(D) \). To find the PDF of \( Y \), we apply the change of variables formula. Let \( g(D) = \exp(D) \). Since \( Y = g(D) \) and \( D = g^{-1}(Y)= \ln Y \), then 
\begin{align*}
f_Y(y) &= f_D(g^{-1}(y)) \left| \frac{d}{dy} g^{-1}(y) \right| \\
&= \frac{1}{\tilde{\sigma} \sqrt{2\pi}} \exp\left( -\frac{(\ln y - \mu)^2}{2\tilde{\sigma}^2} \right)\left|\frac{d}{dy} \ln(y) \right| \\
&=\frac{1}{y \tilde{\sigma} \sqrt{2\pi}} \exp\left( -\frac{(\ln y)^2}{2\tilde{\sigma}^2} \right),
\end{align*}
for \( y > 0 \).
}

We simplify the RHS of (\ref{eq:X_and_a}) by expressing it as a function, $X(y)$, of $y$ as
\[
X(y) = \frac{y}{1 + h(y - 1)}.
\]
Now, $y$ can be expressed in terms of $X$ according to
\[
y = \frac{X(y)}{1 - hX(y) + h}.
\]
Solving $X(y) = x$ for $y$ gives
\begin{equation}
y = \frac{x}{1 - hx + h},
\label{eq:yandx}
\end{equation}
leading to the derivative
\begin{equation}
\frac{dy}{dx} = \frac{1 + h}{(1 - hx + h)^2}.
\label{eq:derv}
\end{equation}
The PDF of $X$ can be derived by computing the Jacobian of the transformation from $y$ to $X(y)$. Let $y = g(x)$ be the inverse transformation of $X(y)$, and let $J_g(x)$ be the Jacobian matrix of the transformation $g(x)$. By employing the change-of-variables rule, the PDF of $X$ is computed as
\begin{align}
	f_X(x) &= f_Y(y) \left|\det \left(J_g(x)\right)\right| \nonumber \\
&= f_Y(y) \left|\frac{dy}{dx}\right|.
\label{eq:fXx}
\end{align}
Substituting the RHS of (\ref{eq:yandx}) and (\ref{eq:derv}) in (\ref{eq:fXx}) gives
\begin{equation}
f_X(x) = f_Y\left(\frac{x}{1 - hx + h}\right) \left|\frac{1 + h}{(1 - hx + h)^2}\right|.
\label{eq:fX}
\end{equation}
After substituting the PDF of $y$ in (\ref{eq:fX}), the PDF of $b_{j,i}= X$ is obtained as
\begin{align*}
f_X(x) &= \frac{1 + h}{(1 - hx + h)^2} \frac{1}{\frac{x}{1 - hx + h} \tilde{\sigma} \sqrt{2 \pi}} \exp\left(-\frac{\left(\ln \frac{x}{1 - hx + h}\right)^2}{2\tilde{\sigma}^2}\right) \nonumber \\
&= \frac{1 + h}{\tilde{\sigma} \sqrt{2 \pi}} \frac{1}{x (1 - hx + h)} \exp\left(-\frac{\left(\ln \frac{x}{1 - hx + h}\right)^2}{2\tilde{\sigma}^2}\right).
\end{align*}
{
Let \( X_1, \ldots, X_{|\mathcal{D}|} \) be independent and identically distributed (i.i.d.) r.v.s with PDF \(f_X(x) \), where a value in the set $\mathcal{D}=\{D_1, \ldots, D_{|\mathcal{D}|}\}$ is used for generating the corresponding value in \( X_1, \ldots, X_{|\mathcal{D}|} \) using (\ref{eq:X_and_a}). Let \( F_X(x) \) denote the corresponding cumulative distribution function (CDF). We denote \(M = \max(X_1, \ldots, X_{|\mathcal{D}|}) \), whose CDF is \( F_M(m) = [F_X(m)]^{|\mathcal{D}|} \), since the maximum of i.i.d. variables is less than or equal to \( m \) if and only if all individual variables are less than or equal to \( m \). After differentiating, we obtain the PDF of \( M \) as
\[
f_M(m)= \frac{d}{dm}F_M(m)= |\mathcal{D}| f_X(m) [F_X(m)]^{|\mathcal{D}|-1}.
\]
Now, let \(b_{\max} = lM\) be a scaled version of the maximum, where \( l > 0 \) is a constant. By applying max-order statistics, and using a change of variables \( b = l m \), we obtain the PDF of \(b_{\max}\) as
\begin{align}
f_{b_{\max}}(b)&= \frac{1}{l} f_M\left( \frac{b}{l} \right) \nonumber \\
&=\frac{|\mathcal{D}|}{l} f_X\left( \frac{b}{l} \right) \left[ F_X\left( \frac{b}{l} \right) \right]^{|\mathcal{D}|-1} \nonumber \\
&= \frac{|\mathcal{D}|}{l} \frac{1 + h}{\tilde{\sigma} \sqrt{2 \pi}} \frac{1}{b/l (1 - h b/l + h)} \nonumber \\
   	&~~~~ \times \exp\left(-\frac{\left(\ln \frac{b/l}{1 - h b/l + h}\right)^2}{2\tilde{\sigma}^2}\right) \left[ F_X\left( \frac{b}{l} \right) \right]^{|\mathcal{D}|-1}, 
\label{eq:fbmax_bound}
\end{align}
where $b,l > 0$. Since the CDF $F_X\left( \frac{b}{l} \right)$ does not have a closed form solution, we obtain an approximation $\hat{f}_{b_{\max}}(b)$ of $f_{b_{\max}}(b)$ by estimating $F_X\left( \frac{b}{l} \right)$ via Monte Carlo simulation for $|\mathcal{D}|=100$.

}
\begin{figure}[h]
  \centering
  \includegraphics[scale=0.45]{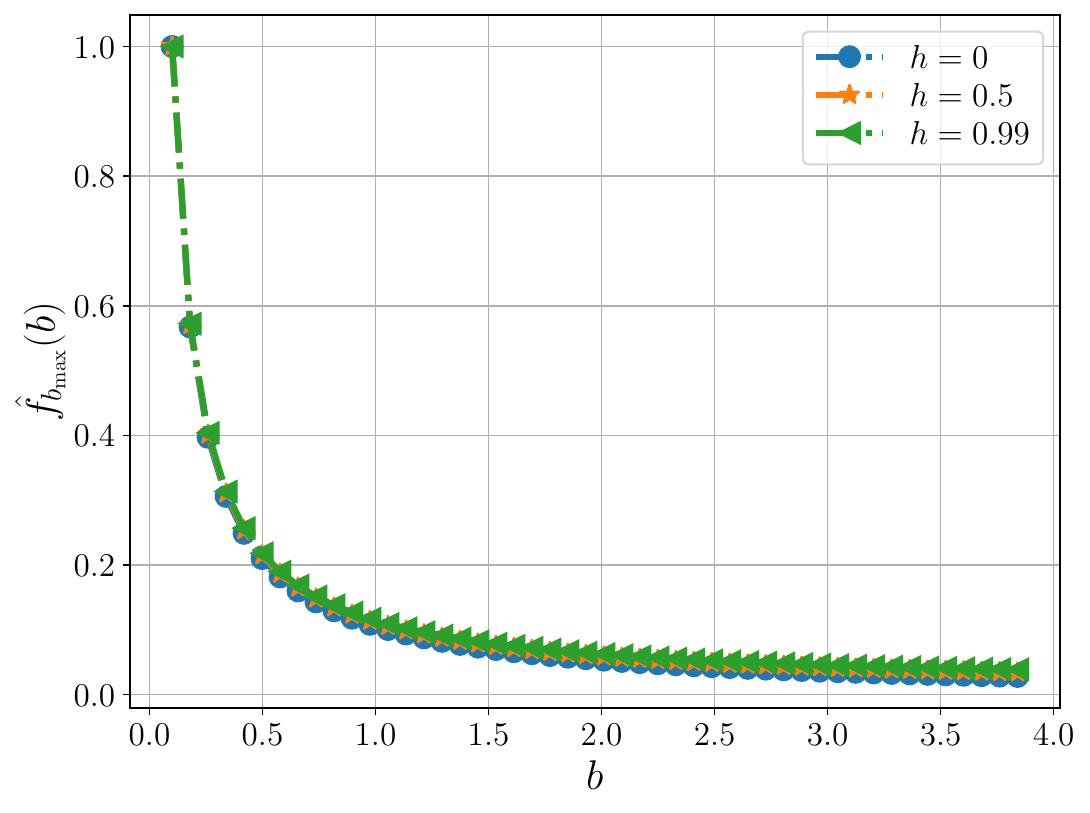}
  \vspace{-0.12cm}
  \caption{Plot of $\hat{f}(b_\max)$ vs. $b$ for different values of $h$, using an example in which $N_i \sim \mathcal{N}(0,1)$, $W_{j,i}\in [-10,10]$, $n=14$, and $l=7$.}
\label{fig:pdf_plot} 
\end{figure}

The plot of $\hat{f}(b_\max)$ is shown in Fig.~\ref{fig:pdf_plot} for different values of $h$. We observe that the bound only changes slightly as $h$ changes. Therefore, for simplicity, we consider $h=0$ in the remainder of the appendix. The hyper-parameters of this simulation were chosen to match the training data dimension of the studied Fitbit fitness tracking dataset.

\noindent Note that
\begin{align}
F_{b_\max}(1)&= \int_{0}^{1} f_{b_\max}(b) db \nonumber \\
&=\Pr[b_\max\leq 1].
\end{align}
We evaluate \( F_{b_\max}(b)\) numerically using the trapezoidal rule to integrate a discrete PDF according to
\[
F_{b_\max}(b) \approx \sum_{i: x_i \leq b} \frac{\hat{f}_{b_\max}(x_{i+1}) + \hat{f}_{b_\max}(x_i)}{2} \Delta x,
\]
where \(x_i\ \) represents a discrete value of \( b_\max \), and \( \Delta x=|x_{i+1}-x_i| \).


Using (\ref{eq:ineq2}) we infer that 
\begin{align}
\Pr\left[\bar{\Delta}_{\mathrm{SPOF}}^{[N]}(i)< \bar{\Delta}_{\mathrm{SPOF}}(i)\right]\geq F_{b_\max}(1).
\label{eq:prob_bmin2}
\end{align}


\begin{figure}[H]
  \centering
  \includegraphics[scale=0.45]{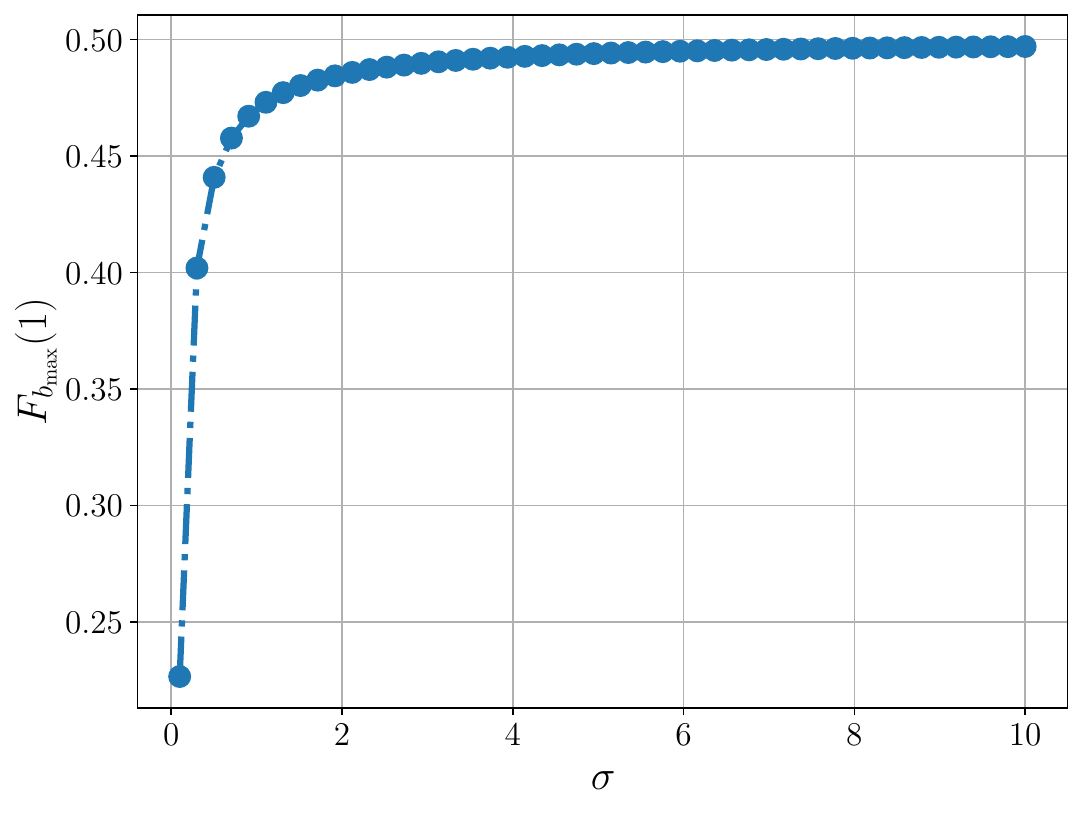}
  \vspace{-0.12cm}
  \caption{Plot showing a lower-bound of $\Pr\left[\bar{\Delta}_{\mathrm{SPOF}}^{[N]}(i)< \bar{\Delta}_{\mathrm{SPOF}}(i)\right]$ for different values of environmental noise parameter $\sigma$ for $l=7$ used in our Fitbit fitness tracking dataset.}
\label{fig:prob} 
\end{figure}
The plot the RHS in (\ref{eq:prob_bmin2}) for different values of environmental noise s.d. $\sigma$ is shown in Fig.~\ref{fig:prob}. 

\begin{rem}
We infer from Fig.~\ref{fig:prob} that with probability at least $0.5$, SPOF in {the} presence of noisy inputs require less DP noise to achieve the same level of privacy as SPOF with noiseless inputs, provided the  level of environmental noise is sufficiently large. 
\label{rem:SPOF_noise}
\end{rem}


\section{Comparison of SPOF and DP-SGD Perturbation Complexities}
\label{app:complexity}
We compare the computational complexities of {Steps 3 and 5} in Algorithm 1 with {Steps 4 and 6} in Algorithm 2, taking into account the relevant arithmetic operations involved in these portions of the algorithms. Let $O_{\mathrm{ADD}}, O_{\mathrm{SUB}}, O_{\mathrm{MUL}}$, and $O_{\mathrm{DIV}}$, denote the number of macro operations for a single addition, subtraction, multiplication, and division, respectively. For example, in case of an AMD K7 processor, we have $O_{\mathrm{ADD}}=1, O_{\mathrm{SUB}}=1, O_{\mathrm{MUL}}=3$, and $O_{\mathrm{DIV}}=32$ \cite{Fog}. In the case of SPOF,
\begin{itemize}
	\item Perturbation of loss function coefficients $\alpha_{j,i,2}$ and $\alpha_{j,i,3}$ require two additions for each $i\in\{1,\ldots,n\}$.  
	\item To improve computation efficiency in hardware, noise is typically generated by pre-fetching precomputed noise values from memory \cite{noise_comp}. Reading a noise sample involves two additions, a 2:1 multiplexer (which consists of one NOT, two AND and one OR gate), and a bit shifter which is unused in DP (see Fig.  2a of \cite{noise_comp}). Note that in this work we do not consider any additional privacy mechanism such as noise scrambling as shown in Fig.  2b of \cite{noise_comp}. Thus, the complexity, $C_N$, of noise generation is given by 
\begin{align*}
	 C_N&= 2\times O_{\mathrm{ADD}}+1\times O_{\mathrm{NOT}}+2\times O_{\mathrm{AND}}+1\times O_{\mathrm{OR}} \\
	 &=6.
\end{align*}
	\item Coefficient $\alpha_{j,i,2}$ involves a single subtraction, and $\alpha_{j,i,3}$ involves one multiplication and a subtraction.
	\item Additionally, applying loss stabilization constant in {Step 3} requires $l$ additions for each $i\in \{1,\ldots,n\}$, \emph{i.e.,} $nl$ total additions.
\end{itemize}
Thus, the complexity incurred by {Steps 3 and 5} of Algorithm 1 per user is expressed as
\begin{align*}
	 C_{\mathrm{SPOF}}&= (2n\times O_{\mathrm{ADD}}+2nC_N) \text{ (perturbations)}\nonumber \\
   	&~~~ + (2n\times O_{\mathrm{SUB}}+ n \times O_{\mathrm{MUL}}) \text{ (coefficients)} \\
	 &~~~~~ +(nl\times O_{\mathrm{ADD}}) \text{ (loss stabilization)}
	  \nonumber \\
	 &= (19+l)n.
\end{align*}
On the other hand, in the case of DP-SGD,
\begin{itemize}
   \item The gradients $\nabla \hat{L}_{j,i}~\forall i\in\{1,\ldots,n\}$ need to be perturbed, which requires $n$ additions of noise. 
   \item Gradient clipping in {Step 4} of Algorithm 2 involves $n$ divisions, one for each $\nabla \hat{L}_{j,i}$ update, and another division for computing $\frac{\lVert \nabla \hat{L}_j \lVert_2}{C}$.
   \item The $A_{j,i}$ term in (\ref{eq:ami}) requires 
   \begin{align*}
   	C&= O_{\mathrm{ADD}} \text{ (additions)}+ O_{\mathrm{SUB}} \text{ (subtractions)}\nonumber \\
   	&~~~ + O_{\mathrm{DIV}} \text{ (divisions)}+ O_{\mathrm{exp}}~(\mathrm{exp}~\text{operations)}
\end{align*}
macro operations for an update. Since in hardware an $\mathrm{exp}$ function utilizes table lookup, it has negligible complexity and the corresponding number of macro operations $O_{\mathrm{exp}}=0$. Hence, we obtain
\begin{align*}
   	C&= O_{\mathrm{ADD}} \text{ (additions)}+ O_{\mathrm{SUB}} \text{ (subtractions)}\nonumber \\
   	&~~~ + O_{\mathrm{DIV}} \text{ (divisions)}\nonumber \\
   	&= 34.
\end{align*}
	\item Additionally, other than computing $A_{j,i}$ itself, calculating $\lVert \nabla \hat{L}_j \lVert_2=\sqrt{A_{j,1}^2+\cdots+A_{j,n}^2}$ requires $n-1$ additions, $n$ squaring operations (which is equivalent to $2n$ multiplications), and one square root operation.
\end{itemize}
An iterative algorithm known as the Newton-Raphson method is commonly used for square root calculations. Each iteration requires a single addition, multiplication and a division, and around $7$ iterations are performed on average \cite{squareroot}. Thus, the total complexity for privatizing the $j$-th training sample using {Steps 4 and 6} of Algorithm 2 is given by
\begin{align*}
	C_{\mathrm{DP-SGD}}&= (n \times O_{\mathrm{ADD}}+n \times C_N) \text{ (perturbations)} \nonumber \\
	&~~~ + (n+1)\times O_{\mathrm{DIV}} \text{ (divisions for clipping)} \nonumber \\
	&~~~~~~ +(n-1)\times O_{\mathrm{ADD}} \text{ (additions)} \nonumber \\
	&~~~ + 2n\times O_{\mathrm{MUL}} \text{(squaring)} \nonumber \\
	&~~~ + 7(O_{\mathrm{ADD}}+O_{\mathrm{MUL}}+O_{\mathrm{DIV}}) \text{ (square root)}\nonumber \\
	&~~~ +nC ~(\text{all~}A_{j,i}\text{ terms}) \nonumber \\
	&= 80n+283.
\end{align*}

\section{Optimization of SPOF and DP-SGD Sensitivities}
\label{app:sensitivity_optimize}

In the case of SPOF, {the $i$-th term of (\ref{eq:delFMorig0_2}) is expanded in (\ref{eq:senfmi}),} and its upper-bound is derived in (\ref{eq:delfmi}). 
\begin{figure*}[h]
\begin{align}
	S_{\mathrm{SPOF}}(i)&\triangleq \max_{x_{j,i},x_{j,i}'} \sum_{k=1}^2 \sum_{r=0}^2 \left| f_{j,i,k}^{(r)}(a)-  f_{j,i,k}^{(r)'}(a) \right| \nonumber \\ 
	&\leq 2\max_{x_{j,i}} \sum_{k=1}^2 \sum_{r=0}^2 \left| f_{j,i,k}^{(r)}(a)\right|  \label{eq:senfmi}\\ 
	&= 2\max_{x_{j,i}} \left(\left|f_{j,i,1}^{(0)}(a) \right|+\left|f_{j,i,1}^{(1)}(a) \right|+\left|f_{j,i,1}^{(2)}(a) \right|+\left|f_{j,i,2}^{(0)}(a) \right|+\left|f_{j,i,2}^{(1)}(a) \right|+ \left|f_{j,i,2}^{(2)}(a) \right| \right) \nonumber \\
	&= 2\max_{x_{j,i}} \left(\left|x_{j,i} \log(1+e^{-a}) \right| +\left|-\frac{ x_{j,i}}{1+e^{a}} \right|+\left|x_{j,i} \frac{e^{a}}{(1+e^{a})^2} \right|+\left|(1-x_{j,i}) \log(1+e^{a}) \right| \right. \nonumber \\
	&~~~~~ \left.  +\left|\frac{ (1-x_{j,i})}{1+e^{-a}} \right|+\left|-(1-x_{j,i}) \frac{e^{-a}}{(1+e^{-a})^2} \right| \right) \nonumber \\
	&= 2\max_{x_{j,i}} \left(x_{j,i} \log(1+e^{-a})+ \frac{ x_{j,i}}{1+e^{a}} + x_{j,i} \frac{e^{a}}{(1+e^{a})^2}+ (1-x_{j,i}) \log(1+e^{a})\right. \nonumber \\
	&~~~~~\left. + \frac{ (1-x_{j,i})}{1+e^{-a}} + (1-x_{j,i}) \frac{e^{-a}}{(1+e^{-a})^2} \right) \nonumber \\
	&= \log(1+e^{a})+2\max_{x_{j,i}} \left(x_{j,i} \frac{\log(1+e^{-a})}{\log(1+e^{a})} + (1+x_{j,i})\frac{e^{a}}{(1 + e^{a})^2} +\frac{e^{a}}{1 + e^{a}}\right)  \nonumber \\
	&= \log(1+e^{a})+2\left(\frac{\log(1+e^{-a})}{\log(1+e^{a})} + \frac{2e^{a}}{(1 + e^{a})^2} +\frac{e^{a}}{1 + e^{a}}\right)  \nonumber \\
	&\triangleq \Delta_{\mathrm{SPOF}}(i).
	\label{eq:delfmi} 
\end{align}
\end{figure*}
On the other hand, in the case of DP-SGD, the randomizing mechanism's sensitivity depend on the clipped gradient in (\ref{eq:clippedgrad}) of the DA loss function in (\ref{eq:loss_jl_hat}). 
For the $i$-th feature of the $j$-th training sample, we define the sensitivity as 
\begin{align}
S_{\mathrm{DP-SGD}}(i)&\triangleq \max_{x_{j,i},x_{j,i}'} \left(\left|\nabla \bar{L}_{j,i}- \nabla \bar{L}_{j,i}'\right|\right) \nonumber \\
&\leq 2\max_{x_{j,i}} \left(\left|\nabla \bar{L}_{j,i}\right|\right) \nonumber \\
&= 2 \max_{x_{j,i}} \left|\frac{\frac{e^a}{1+e^a}-x_{j,i}}{\max\left(1, \frac{\lVert \nabla \bar{L}_j \lVert_2}{C}\right)} \right| \nonumber \\
&\triangleq \Delta_{\mathrm{SGD}}(i),
\end{align}
where $C$ is {clipping threshold}. 
Suppose that 
\begin{equation}
A_{j,i}\triangleq\frac{e^a}{1+e^a}-x_{j,i},
\label{eq:ami}
\end{equation}
\begin{align}
T= \left(\frac{e^a}{1+e^a}-x_{j,1}\right)^2+\cdots+\left(\frac{e^a}{1+e^a}-x_{j,n}\right)^2,
\label{eq:T}
\end{align}
and $B\triangleq T-A_{j,i}^2$. Then,
\begin{align}
\lVert\nabla \hat{L}_j \lVert_2= \sqrt{A_{j,i}^2+B},
\end{align}
and we can write
\[
\Delta_{\mathrm{SGD}}(i)= 2 \max_{A_{j,i}} \left|\frac{A_{j,i}}{\max\left(1, \frac{\sqrt{A_{j,i}^2+B}}{C}\right)} \right|.
\]
Note that
\[
\max\left(1, \frac{\sqrt{A_{j,i}^2+B}}{C}\right) =
\begin{cases}
1, & \text{if } \sqrt{A_{j,i}^2+B} < C, \\
\frac{\sqrt{A_{j,i}^2+B}}{C},  & \text{otherwise}.
\end{cases}
\]
Thus, we can rewrite 
\begin{equation}
\Delta_{\mathrm{SGD}}(i)= 
\begin{cases}
2 \max_{A_{j,i}}|A_{j,i}|, & \text{if } \sqrt{A_{j,i}^2+B} < C, \\
2 \max_{A_{j,i}}\left|\frac{A_{j,i} C}{\sqrt{A_{j,i}^2+B}}\right|, & \text{otherwise}.
\end{cases}
\label{eq:delta_sgdi}
\end{equation}
Moreover, as $\sqrt{A_{j,i}^2+B}\geq 0$ and $C\geq 0$, we can rewrite (\ref{eq:delta_sgdi}) as
\begin{equation}
\Delta_{\mathrm{SGD}}(i)= 
\begin{cases}
2 \max_{A_{j,i}} |A_{j,i}|, & \text{if } \lVert\nabla \hat{L}_j \lVert_2 < C, \\
2 \max_{A_{j,i}}\frac{|A_{j,i}| C}{\lVert\nabla \hat{L}_j \lVert_2}, & \text{otherwise}.
\end{cases}
\label{eq:delta_sgdi2}
\end{equation}
Let $\Delta_{\mathrm{SGD(1)}}(i)=\Delta_{\mathrm{SGD}}(i)$ when $\lVert\nabla \hat{L}_j \lVert_2 < C$. More specifically, we can write
\begin{align}
\Delta_{\mathrm{SGD(1)}}(i)&=  2\max_{x_{j,i}} \left|\frac{e^a}{1+e^a}-x_{j,i}\right| \nonumber \\
&=\frac{2e^a}{1+e^a}.
\label{eq:delta_sgd1}
\end{align}
Let $\Delta_{\mathrm{SGD(2)}}(i)=\Delta_{\mathrm{SGD}}(i)$ when $\lVert\nabla \hat{L}_j \lVert_2 \geq C$, where $B$ and $C$ are constants. We then obtain
\begin{align}
\Delta_{\mathrm{SGD(2)}}(i)&= 2 \max_{A_{j,i}}\frac{|A_{j,i}| C}{\sqrt{A_{j,i}^2+B}}. \nonumber \\
&=
\begin{cases}
	2 \max_{A_{j,i}}\frac{A_{j,i}C}{\sqrt{A_{j,i}^2+B}},\text{ if }A_{j,i}\geq 0, \\
	2 \max_{A_{j,i}}\frac{-A_{j,i}C}{\sqrt{A_{j,i}^2+B}},\text{ otherwise}.
\end{cases}
\label{eq:delta_sgd2}
\end{align}


First consider $A_{j,i}\geq 0$, and let $y=\frac{A_{j,i} C}{\sqrt{A_{j,i}^2+B}}$. The maximum value of $y$ can be obtained by solving for $\frac{dy}{dA_{j,i}}=0$ {which} gives
\begin{align*}
	\frac{dy}{dA_{j,i}}&= \frac{-A_{j,i}^2C}{(A_{j,i}^2+B)^{3/2}}+\frac{C}{(A_{j,i}^2+B)^{1/2}} \nonumber \\
	&= \frac{CB}{(A_{j,i}^2+B)^{3/2}}.
\end{align*}
When $\frac{dy}{dA_{j,i}}=0$, $CB=0$. Since $C\neq 0$, it follows that $B$ must be zero. Similarly, when $A_{j,i}< 0$, let $y=\frac{-A_{j,i} C}{\sqrt{A_{j,i}^2+B}}$. It can be verified that  
\begin{align*}
	\frac{dy}{dA_{j,i}}&= \frac{A_{j,i}^2C}{(A_{j,i}^2+B)^{3/2}}-\frac{C}{(A_{j,i}^2+B)^{1/2}} \nonumber \\
	&= \frac{-CB}{(A_{j,i}^2+B)^{3/2}}.
\end{align*}
When $\frac{dy}{dA_{j,i}}=0$, $CB=0$, and it again follows that $B$ must be zero. Thus, by substituting $B=0$ in (\ref{eq:delta_sgd2}) we get
\[\Delta_{\mathrm{SGD(2)}}(i)= 2C\]
as {the} sensitivity cannot be negative. This result shows that $S_{\mathrm{DP-SGD}}$ is upper-bounded by $2C$ when $\lVert\nabla \hat{L}_j \lVert_2 \geq C$. 

\begin{figure}[H]
  \centering
  \includegraphics[scale=0.45]{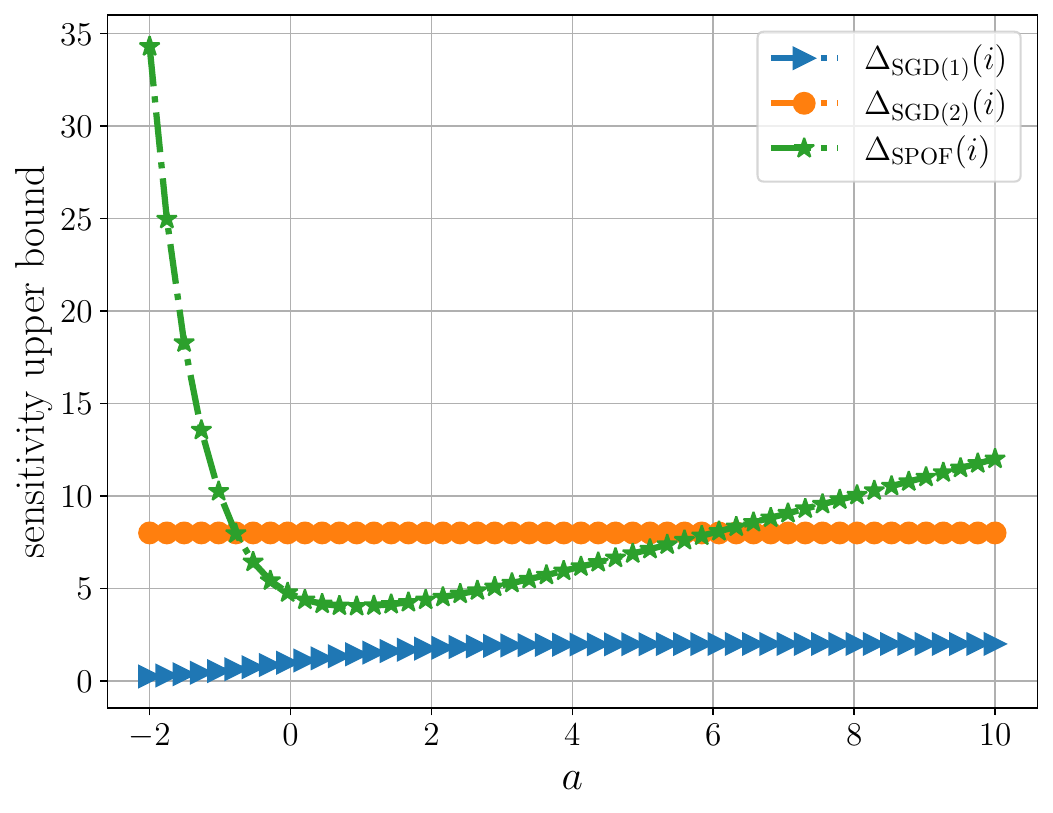}
  \vspace{-0.12cm}
  \caption{Comparison of SPOF and DP-SGD sensitivity upper-bounds for different values of $a$. We choose $C=4$ for computing $\Delta_{\mathrm{SGD(2)}}(i)$.}
\label{fig:sensitivity1} 
\end{figure}

In Fig.~\ref{fig:sensitivity1} the values of $\Delta_{\mathrm{SGD(1)}}(i)$ and $\Delta_{\mathrm{SPOF}}(i)$ are shown for different values of $a$. It can be verified that, within the range of $a$ in Fig.~\ref{fig:sensitivity1}, $\min \Delta_{\mathrm{SPOF}}(i)\approx 4.0348$ occurs at $a\approx 0.9057$, whereas $\Delta_{\mathrm{SPOF}}(i)\approx 4.6931$ when $a=0$, with the discrepancy between these values being $16.3\%$ of $\min \Delta_{\mathrm{SPOF}}(i)$. This highlights that $\Delta_{\mathrm{SPOF}}(i)$ is not very far from its minimum when $a=0$, making this choice useful for SPOF as it reduces the number of coefficients to be perturbed by $33\%$. Fig.~\ref{fig:sensitivity1} reveals that $\min \Delta_{\mathrm{SGD(1)}}(i)<4.6931$ when $a=0$, but $\Delta_{\mathrm{SGD(2)}}(i)$, computed for $C=4$ used in the simulations shown in Fig.~\ref{fig:res1}, is significantly larger. Additionally, Fig.~\ref{fig:sensitivity1} suggests that, for a fixed privacy budget $\epsilon$, DP-SGD will have a lower utility than SPOF if its noise scale depends only on $\Delta_{\mathrm{SGD(1)}}(i)$.

\end{appendices}

\end{document}